\documentclass[journal,10pt]{IEEEtran}
\IEEEoverridecommandlockouts      
\author{Shuang~Gao,~\IEEEmembership{Member,~IEEE}, and Peter E. Caines,~\IEEEmembership{Life Fellow,~IEEE}\quad (Version: August 5, 2022)
\thanks{Shuang Gao and Peter E. Caines are with the Department of Electrical and Computer Engineering, McGill University,
   Montreal, QC, Canada \hspace{1cm}
         {\tt\small    $\{$sgao,peterc$\}$@cim.mcgill.ca}.}  
}%

\usepackage[noadjust]{cite}

\usepackage{graphics}
\usepackage{epstopdf}

\usepackage{amsmath}
  \usepackage[font=footnotesize]{subfig}

\usepackage{hyperref}
\usepackage{enumerate}
\usepackage{color}

\newcommand{\BR}{ \operatorname{R}}  
\newcommand{\BN}{\mathds{N}}

\newcommand*\TRANS{{\mathpalette\doTRANS\empty}}
\makeatletter
\newcommand*\doTRANS[2]{\raisebox{\depth}{$\m@th#1\intercal$}}
\makeatother

\usepackage{url}
\usepackage{amssymb,mathtools}

\usepackage[standard,thmmarks,amsmath]{ntheorem}
\usepackage{times}
\usepackage{dsfont}
\usepackage{tikz}

\begin{document}
%
\title{Transmission Neural Networks:\\ From  Virus Spread Models to Neural Networks}
%
%
%
\maketitle 
\begin{abstract}
This work connects models for virus spread on networks with their equivalent neural network representations. 
Based on this connection, we propose a new neural network architecture, called Transmission Neural Networks (TransNNs) where activation functions are primarily associated with links and are allowed to have different activation levels. 
Furthermore, this connection leads to the discovery and the derivation of three new activation functions with  tunable or trainable parameters.  
Moreover, we prove that TransNNs with a single hidden layer and a fixed non-zero bias term  are universal function approximators. Finally, we present new fundamental derivations of continuous time  epidemic network models based on TransNNs.
\end{abstract}


\section{Introduction}
\subsection{Motivation and Introduction}
Epidemic models are important in understanding epidemic spread dynamics and designing strategies to control virus spread. Researchers have proposed a variety of epidemic models and devoted much effort to analyze these models (see  \cite{pastor2015epidemic,nowzari2016analysis,pare2020modeling} 
and the references therein), but there is still a lack of satisfactory fundamental understandings and derivations of these models on heterogeneous transmission networks. %
On a seemingly unrelated topic,  neural network models as universal function approximators are very successful in fitting input-output data relations in practice (see e.g. \cite{lecun2015deep,goodfellow2016deep}); however, fundamental understandings and interpretations of parameters, models and their outputs are still under ongoing research efforts. 
This work studies the virus spread process on networks based on first principles analysis  and establishes its equivalent neural network representation. This equivalence (a) leads to a useful representation of epidemic models via neural networks and (b)  improves our  fundamental understandings of neural network models. Although the starting point is the virus spread model,  the model can be adjusted to characterize other problems with similar  dynamics, including computer virus infections, rumour's spread, neuronal excitation propagation, among others.  One salient feature of such dynamics is that viruses (or rumours, or neuronal excitations) replicate themselves on nodes of transmission networks and no conservation of flow is required, which make them fundamentally different from power transmission networks and  flow networks.%

 Transmission Neural Networks (TransNNs) are proposed in this paper to  represent the neural network models with learnable parameters where 1) the underlying connections among nodes of the network represent independent transitions of particles\footnote{Depending on the application contexts, the particles in this paper may be interpreted as  neurotransmitters, viruses (or droplets), passengers, etc.} from one node to another;  
2) the effectiveness of transitions of particles are assumed to be independently random with probability weights;
 3) the transmission of particles changes the probability of the nodal state (which is  binary) following  virus spread dynamics.

The virus spread dynamics models in this paper contain two types: (I) single virus particle transmission and (II)  multiple virus particle transmission.  
 Each of the two types of virus transmission models has two equivalent representations: the virus spread representation and the TransNN representation.
In the virus spread representations,  the probability of being infected is used to represent the state of each node.
 In the TransNN representations, the negative-log-negative probability (which corresponds to the Shannon information of being healthy) is used to represent the state of each node. 
\subsection{Related Literature}

The first  epidemic model with heterogeneous networks  is proposed by Lajmanovich and Yorke in \cite{lajmanovich1976deterministic}  where the network characterizes  transmission links and probabilities among heterogeneous populations.  
 Kephart and White analyzed the virus spread model on random directed graphs generated with homogenous edge-connection  probability in \cite{kephart1992directed}. Virus spread models on the networks characterized by degree distributions have been analyzed  in \cite{pastor2001epidemic}.  
Virus spread models characterizing the probability of infection for networks with uniform transmission probabilities on all links have been proposed  and used for identifying threshold values for epidemics (to die out) in the work \cite{wang2003epidemic} and its later modification and extension \cite{chakrabarti2008epidemic}.
Important mean field approximation results for virus spread models on networks have been established in \cite{van2008virus,cator2012second,van2015accuracy}, where mean field states are interpreted as approximations of the fractions of the infected in nodal populations. Another type of related models with different dynamics includes the message-passing networks in \cite{karrer2014percolation,bianconi2014multiple}  which have been utilized to study epidemic spread control \cite{altarelli2014containing} and   identify influential nodes \cite{morone2015influence}.
For an overview of other existing epidemic models, readers are referred to  \cite{pastor2015epidemic,nowzari2016analysis,pare2020modeling} and the references therein.
The virus spread models used in the current paper are essentially generalized versions of that in \cite{chakrabarti2008epidemic}  
  to network models with heterogeneous infection probabilities among links.
In such models, it is assumed that the successful transmissions of particles  (which may be interpreted as viruses, passengers or neurotransmitters depending on the problem contexts) 
are independent among different transmission links reminiscent of the conditional independence assumption in Na\"ive Bayes  \cite{webb2010naive}. 

TransNN models proposed in this paper may be used as universal function approximators and deep learning models. 
The key conceptual difference between TransNNs and standard neural network architectures (see e.g. \cite{goodfellow2016deep}) is that activation functions and their activation levels  in TransNNs  are considered as a part of links, whereas in standard neural networks (whether they are feedforward or recurrent), activation functions are typically considered as a part of nodes (or the neuron-like computing units). 

Belief networks (also known as Bayesian networks, influence diagrams, causal networks, decision networks, or relevance diagrams) model the probability distribution based on conditional dependence among attributes (see e.g. \cite{jensen1996introduction, pearl1988probabilistic}).  {These models are  different from  TransNN models,  since  belief networks represent the conditional dependence whereas networks in TransNNs represent contact structures and the probability of connections. Nevertheless, if we take time steps into considerations for TransNNs, then the conditional dependence presents itself in the forward time direction.} 

As inference and learning models for probability distributions, belief networks with directed and acyclic connections called (sigmoid) deep belief networks are proposed in \cite{neal1990learning}, and Markov networks with symmetric connections called Boltzmann machines  are studied in \cite{sherrington1975solvable,hinton1983optimal}. %
 A fast algorithm for training deep belief networks is proposed by Hinton, Osindero, and Teh in \cite{hinton2006fast} which utilizes restricted Boltzmann machines (RBM) (i.e.  Boltzmann machines restricted to bipartite graphs \cite{hinton2002training, smolensky1986information}). %
 Universal approximation properties for a specific type of deep belief networks are established in \cite{sutskever2008deep}.  
Although as inference models, TransNNs with stochastic states (or random realizations of states) for each node may  potentially be viewed as deep belief network models or Boltzmann machines,  TransNNs differ from these models in the connection functions among nodes.

 There has been a recent surge of research interests and  efforts in studying activation functions with trainable or tunable parameters and shapes (see \cite{apicella2021survey} and the references therein). The three activation functions identified in this work (called TLogSigmoid, TLogSigmoidPlus and TSoftAffine) are trainable or tunable  and differ from all the existing activation functions in the literature (see e.g. \cite{glorot2011deep,dugas2001incorporating,clevert2015fast,maas2013rectifier,apicella2021survey,berner2021modern}). Furthermore, special cases of these  activation functions can be related to the Log-Sigmoid function \cite{polyak2001log}, the ReLU function \cite{glorot2011deep} and the Softplus function \cite{dugas2001incorporating}, and special cases of their derivatives can be associated with the sigmoid function and the tanh function. 
%
%
\subsection{Contribution}
This paper proposes a new neural network architecture called Transmission Neural Networks (TransNNs),  establishes equivalent characterizations of virus spread models via TransNNs  using the negative-log-negative transformation of the probability states, and discoverers  a new set of tunable activation functions. 
Furthermore, with additional assumptions on the transmission rate with respect to the time duration,  we provide a new fundamental  derivation of the standard network Susceptible-Infected-Susceptible (SIS) model characterized by differential equations (\cite{lajmanovich1976deterministic, van2008virus,van2015accuracy}) based on TransNNs, and such a derivation (which the authors haven't been able to find in the literature) deepens our understandings of network SIS epidemic models. 

\subsection{Organization}
The paper is structured as follows.
We first present the virus spread model over deterministic effective transmission networks in Section \ref{sec:deterministic-TransNN}. Then we present the dynamics models on  probabilistic transmission networks with a single particle transmission across each link and derive the TLogSigmoid activation function in Section \ref{sec:single-probablistic-transmission}. 
Afterward, we introduce the probabilistic network dynamics with multiple particle transmission across each link in Section \ref{sec:multi-probabilistic-trans}. For each type of the virus spread dynamics, we derive the equivalent TransNN models  in their respective sections. In Section \ref{sec:Psi-Activation},  we investigate in detail the properties of three new activation function (i.e. TLogSigmoid, TLogSigmoidPlus and TSoftAffine).  In Section \ref{sec:TransNNs}, we present TransNNs in their general forms.  We then prove the  universal approximation property for feedforward TransNNs with one hidden layer in Section \ref{sec:universal-approx-TransNNs}. Finally, we present the new derivation of network SIS models based on TransNNs in Section \ref{sec:netowrk-sis-via-TransNN}.
\vspace{0.1cm}
\subsubsection*{Notation and Terminology} 
Let $[n] \triangleq \{1,2,...,n\}$ denote an ordered set. Let $\BN_0$ denote non-negative integers and $\BN$ positive integers. $\operatorname{Q}$ denotes the set of all rational numbers. For a  vector $v \in \BR^n$,
$v_i $ denotes its $i$th element. $[a_{ij}] \in \BR^{n\times n}$ denotes the matrix whose  $ij$th element is specified by $a_{ij}$ for all $i,j \in [n]$. 
We use $\exp_\circ(\cdot): \operatorname{R}^n \to \operatorname{R}^n$ to denote the point-wise exponential function defined by  $\exp_\circ(v)\triangleq [e^{v_1},...., e^{v_n}]^\TRANS$ for any $v=[v_1,..., v_n]^\TRANS \in \operatorname{R}^n $. We use ``neural networks" to refer to   artificial network models  with trainable parameters as function approximators, and ``neuronal networks'' as the networks of biological neurons. As for notations for the three new activation functions, we use (a) $\Psi$ to denote the TLogSigmoid function, (b) $\Psi_+$ to denote the TLogSigmoidPlus function, and (c) to denote $\Phi$ the TSoftAffine function.

\section{Virus Spread dynamics over effective transmission networks} \label{sec:deterministic-TransNN}
Consider a network of multiple persons where person $i \in V$ at time $k\in \BN_0$ has an actual state $x_i(k)$ in $\mathcal{X}=\{0,1\}$ where $0$ and $1$ respectively represent healthy and infected states.  
The probability a person $i$ in the infected state $1$ at time $k$ is denoted by $p_i(k)$, that is,
\[
p_i(k) = \operatorname{Pr}(x_i(k)  = 1),\quad i \in V.
\]
  An \emph{effective transmission link} from person $i$ to person $j$ is defined as the effective transmission of at least one virus from person $i$ to person $j$ that causes the infection of person $j$.  The network of nodal persons with effective transmission links are called the \emph{effective transmission network}. 
   By this definition,  individuals can only affect their immediate neighbours on the effective transmission network.   
Consider a given effective transmission network and
let $(V,E)$ be the underlying (directed) graph with $E\subset V\times V$ where $(i,j) \in E$ if there exists an effective transmission from node $j$ to node $i$, for all $i, j \in V$. Let $A=[a_{ij}]$ be the \emph{adjacency matrix} where for all $i,j \in V$, $a_{ij} =1$ if $(i,j) \in E$ and $a_{ij}=0$ otherwise.  Clearly, all the nodes on the underlying graph $(V,E)$ must have self-loops (that is, $a_{ii}=1$ for all $i\in V$), since every person can effectively transmit virus to himself or herself.   

Without loss of generality, we set $V=[n]$. 
Then the probability of node $i$ being infected at time $k+1$ satisfies 
\begin{equation}\label{eq:1st-det-dyn}
	\begin{aligned}
	(1-p_i(k+1)) &=  \prod_{j  \in N_i^\circ} (1- p_j(k)), \quad i \in [n] 
\end{aligned}
\end{equation}
{where $N_i^\circ  \triangleq  \{j: (i,j)\in E \}$ denotes the neighbourhood of node $i$ with itself included}. It is worth highlighting that the inclusion of the self-loops  is important in the correct characterization of the virus spread dynamics.

We define the following $\log$ function extended by $-\infty$: 
\begin{equation}
	\log (x) \triangleq 
	\begin{cases}
		\ln (x) ,&   x \in (0, 1];\\
		 -\infty, & x =0 .
	\end{cases}
\end{equation}
Let us define the ``\emph{negative-log-negative probability state}" of node $i \in [n]$ as 
\begin{equation}\label{eq:neg-log-neg}
	s_i(k) \triangleq - \log (1-p_i(k)) \in [0,+\infty], \quad k \in \BN_0.
\end{equation}
In information-theoretic terms, {the state $s_i(k)$ is the Shannon information  (also known as the information content or self-information) of the event $x_i(k)=0$ which happens with probability $\operatorname{Pr}(x_i(k)=0)=1-p_i(k)$.}
 The mapping in \eqref{eq:neg-log-neg} from $p_i(k)$ to $s_i(k)$ is monotone, bijective, and concave.  
{Special attention to $\log 0$ must be paid, since for  person $i$ who has been confirmed of infection at time $k\in \BN_0$, $p_i(k)=1$ and $s_i(k)\triangleq-\log(1-p_i(k)) = -\log 0$ is $+\infty$.}

  Taking logarithm and then negation on both sides of \eqref{eq:1st-det-dyn} yields the dynamics for the negative-log-negative probability state as follows:
\[
\begin{aligned}
		s_i(k+1) & = -\sum_{j\in N_i^\circ} \log(1-p_j(k))  \\
		&= \sum_{j\in N_i^\circ} s_j(k) , \quad  s_i(k)  \in [0,+\infty], ~~ k\in \BN_0.
\end{aligned}
\]
Thus, the evolution of the negative-log-negative probability states $s(k) \triangleq (s_1(k)\cdots s_n(k))^\TRANS$ satisfies the linear dynamics
\begin{equation} \label{eq:log-prob-evo}
		s(k+1) =  A s(k), ~~ k\in \BN_0.
\end{equation}
Let's take $s(k)$ as the state of the underlying epidemic spread system  \eqref{eq:1st-det-dyn} over effective transmission networks. Then the probability of infection $p(k)$ can be considered as a nonlinear observation of the underlying state $s(k)$ as follows:
\[
p_i(k)  = 1- e^{-s_i(k)} , \quad i \in [n], ~ k\in \BN_0.
\]
Since the solution to \eqref{eq:log-prob-evo} is given by $s(k) = A^k s(0), $
we obtain 
$$
p_i(k) = 1- e^{-[{A^k s(0)}]_i}, \quad i \in [n], ~ k\in \BN_0.
$$ 
Let $p(k) = [p_1(k),...,p_n(k)]^\TRANS$ and $\exp_\circ(\cdot): \operatorname{R}^n \to \operatorname{R}^n$ denote the \emph{point-wise exponential function} given by  $$\exp_\circ(v)\triangleq [e^{v_1},...., e^{v_n}]^\TRANS, \quad v\in \BR^n.$$  
Then the infection probability vector at time $k$ satisfies
$$
\begin{aligned}
p(k) & = 1 - \exp_\circ{({-A^k s(0)})}	\\
&= 1- \exp_\circ({A^k \log(1- p(0))}), ~~ k\in \BN_0,
\end{aligned}
$$   
and the one-step prediction of the infection probability satisfies  
\begin{equation}\label{eq:one-step-predict}
\begin{aligned}
		p({k+1}) & = 1-\exp_\circ(-{A s(k)}) \\
		&= 1- \exp_\circ({A \log(1- p(k))}),~  k\in \BN_0.
\end{aligned}
\end{equation}
This yields an explicit characterization of the probability of infection at time $k+1$ given the effective transmission network  and the probability of infection $p(k)$ at time $k$.

The effective transmission network  should be obtained and estimated from data. In practice, such data about the effective transmission networks can be gathered from contact tracing (based on, for instance, interviews and location-based check-in systems with scanning QR codes in monitoring the spread of Covid19).  
%
Given the confirmation of infection of some individuals, we can identify and predict nodal individuals that have high probability of infection based on the network reconstructed via contact tracing (i.e. the contact tracing network). 
 It should be noted that although such a contact tracing network  provides partial information  about the underlying effective transmission network, it is clearly not necessarily the effective transmission network. To obtain more accurate information, further infection testing over the contact tracing network is needed. However, when facilities for testing are not available, to minimize the risk of further spread of the virus, we just assume the contact tracing network is the effective transmission network, and inform or isolate nodes with high probabilities of infection. In this case, infection  probabilities computed from contact tracing networks provide upper bounds for the actual infection probabilities derived from effective transmission networks.  

\section{Probabilistic Transmission Networks:\\ Single Particle Transmissions}\label{sec:single-probablistic-transmission}

Let the underlying network represent the \emph{physical contact network}  (which is not necessarily the effective transmission network defined in the previous section), where a link between two persons on a physical contact network exists if the two persons are less than a distance $r$ for a time duration $t$  (e.g., $r=1$ meter and $t=30$s). With a slight abuse of notation, the physical contact network is denoted by $(V,E)$ with the adjacency matrix $A=[a_{ij}]$. Clearly all the nodes of $(V,E)$ must have self-loops, that is, $a_{ii}=1$ for all $i\in V$.  
For all $i,j\in V$, let $w_{ij}$ denote the probability of node $j$ infecting its neighbouring node $i$ on the physical contact network given that $j$ is infected.  

Without loss of generality, we set $V=[n]$.  
Then the probability of node $i$ being infected at time $k+1$ satisfies 
\begin{equation} \label{eq:probablistic-dynamics}
	(1-p_i(k+1)) = \prod_{j\in N_i^\circ} (1- w_{ij} p_j(k)), \quad i \in [n]
\end{equation}
where $N_i^\circ \triangleq  \{j: (i,j)\in E \}$ denotes the neighbourhood of node $i$ with itself included.
{The key modelling assumption in the virus spread dynamics \eqref{eq:probablistic-dynamics} is that, given (or conditioned on) the probabilities of infection at  all nodes, the successful transmissions of the virus are independent among links.} 
The model in \eqref{eq:probablistic-dynamics} when specialized to the case with homogenous cross-node infection probabilities (i.e. $w_{ij}=w$ for $i\neq j$) reduces to  the virus spread model in \cite{chakrabarti2008epidemic}. $w_{ii}$ can be related to  the self-healing probability (denoted by $\delta_i$) via $w_{ii} =1-\delta_i$. 
\begin{remark}[Self-Transmission and Healing Probabilities]
The self-loop weight $w_{ii}$ for node $i \in V$ relates to the self-healing probability denoted by $\delta_i$ via $w_{ii}=1-\delta_i$. 
With $w_{ii}=1$  for all nodes, we essentially have the network Susceptible-Infected (SI) epidemic model where infected nodes stay infected forever.
With $w_{ii}\in (0,1)$ for all nodes, we essentially have the network Susceptible-Infected-Susceptible (SIS) epidemic model where infected nodes can become healthy and may be infected again.
With $w_{ii} =0$ for all nodes, this corresponds to the case where  {nodes after infection heal themselves within one time step.} 
Furthermore, if $w_{ij} =0$ for $i\neq j$, we obtain that 
$
1-p_i(k+1) = 1- w_{ii} p_i(k)
$ and hence 
$
p_i(k) = w_{ii}^k p_i(0).
$
This means that the infection probability decreases geometrically over time  for $w_{ii}\in (0,1)$.  
\end{remark}

\subsection{Equivalent Representations via Neural Networks}

In the following, we present an exact model characterizing the probability evolution in \eqref{eq:probablistic-dynamics} via a neural network model.
First, taking logarithm on both sides of \eqref{eq:probablistic-dynamics} yields
\[
\begin{aligned}
	 \log (1-p_i(k+1)) &= \sum_{j\in N_i^\circ} \log( 1-w_{ij}p_j(k)), ~~i \in [n].\\  
\end{aligned}
\]
Let's introduce the input state and the output state associated with  node $i \in [n]$ as follows: for $k\in \BN_0$,
\begin{equation*}
	\begin{aligned}
		& \text{Input State}:\quad~~ s_i(k)\triangleq - \log(1-p_i(k)),  \\
		& \text{Output State}: \quad o^{h}_i(k) \triangleq - \log (1- w_{hi} p_i(k)),
	\end{aligned}
\end{equation*}
where $i,~ h \in [n].$ 
Strictly speaking the output state is not a nodal state but a link state, since the output of an individual node $i$ depends on the receiving node $h$ as well and more specifically the output state depends on the link probability $w_{hi}$. 

Then the negative-log-negative probability state of the dynamics \eqref{eq:probablistic-dynamics}  satisfies
\[
s_{i}(k+1) = \sum_{j\in N_i^\circ} a_{ij} o^{j}_i (k),  \quad i \in [n]
\]
and the relation between output state and input state is
\begin{equation}\label{eq:typeI-Phi}
	\begin{aligned}
	o^{h}_i(k) & = - \log(1- w_{hi} p_i(k))   \\
	&  = -\log \left[1- w_{hi} (1- e^{-s_i(k)} )\right] \\
	& \triangleq \Psi( w_{hi},s_i(k)), \quad  i \in [n]
\end{aligned}
\end{equation}
where  the activation function $\Psi(\cdot, \cdot)$  takes both the network weight $w_{hi}$ and the input state $s_i$ as inputs. We call the activation function $\Psi(\cdot, \cdot)$ above the \emph{tunable Log-Sigmoid (TLogSigmoid)} activation function. 
An illustration of the  TLogSigmoid activation function $\Psi(w, \cdot)$ with different activation levels specified by $w$ is shown in Fig. \ref{fig:activation-psi1}. 
This provides the flexibility in choosing the activation function with a parameter $w_{ij}$ that governs the level of signal passes at each link $(i,j)\in E$.  We postpone detailed discussions of the TLogSigmoid activation function and its variants to Section~\ref{sec:Psi-Activation}.
\begin{figure}[htb]
\centering 
\includegraphics[width=8cm, trim={3.8cm 0 3.2cm 0},clip]{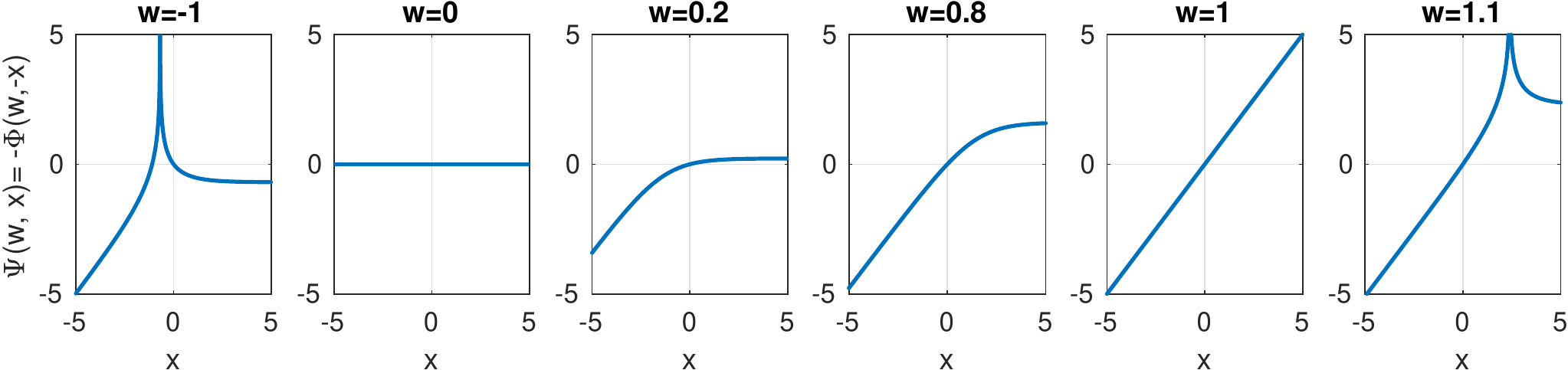}
	\caption{TLogSigmoid activation function $\Psi(w,x)$ with $w$ taking different values in $[0,1]$. When $w=0$, $\Psi(w,\cdot)$ is zero and represents no pass as illustrated by the leftmost picture; when $w=1$, $\Psi(w,\cdot)$ is linear and represents full pass  as illustrated by the rightmost picture; when $w \in(0,1)$, $\Psi(w,\cdot)$ is non-polynomial and monotonically increasing, and represents a partial pass  as illustrated by two figures in the middle. } \label{fig:activation-psi1}
\end{figure}

Therefore, the evolution of  negative-log-negative probability states is equivalently given by
\begin{equation}\label{eq:nn-individual-node-local}
	s_{i}(k+1) = \sum_{j\in N_i^\circ}  \Psi( w_{ij},s_j(k)),~~ i \in [n], ~ k\in \BN_0 
\end{equation}
where $ \Phi( w_{ij},s_j(k)) = - \log \left(1- w_{ij} + w_{ij}e^{ - s_j(k)} \right)$ and  $N_i^\circ \triangleq  \{j: (i,j)\in E \}$ denotes the neighbourhood set of node $i$ with itself included on the underlying graph $([n], E)$.
We call this model the \emph{Transmission Neural Network (TransNN) with single  particle transmissions}. 
 See Fig. \ref{fig:TransNNTypeI-psi} for an illustration of the model above. 
\begin{figure}[htb]
\centering
	\includegraphics[width=8cm]{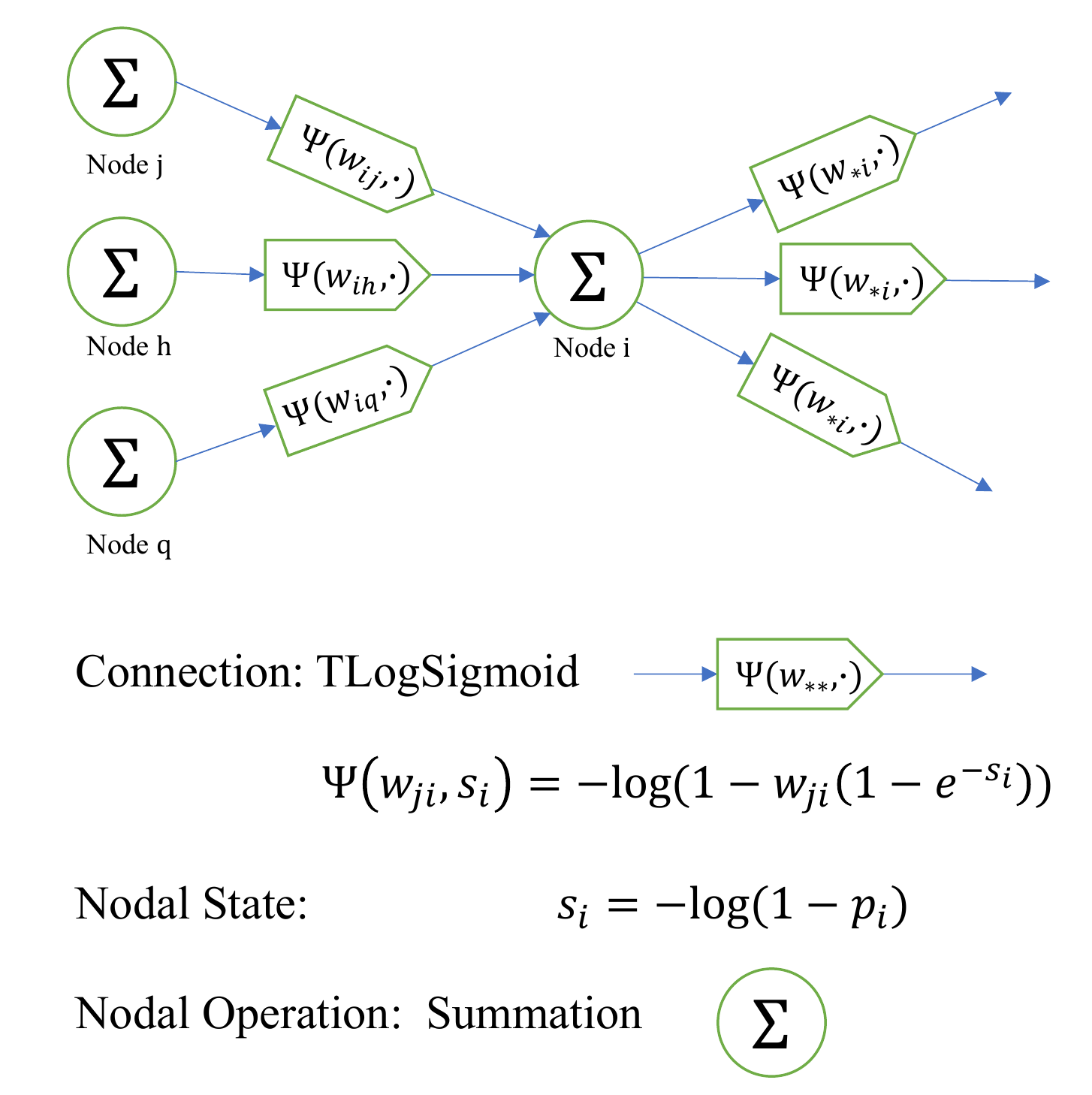}
	\caption{Illustration of TransNN representation of the virus spread network with negative-log-negative probability states and TLogSigmoid activation function. Connections among nodes are nonlinear when $w_{**}\in (0,1)$ and connections may have  different activation levels.} 
	\label{fig:TransNNTypeI-psi}
\end{figure}
We note that $a_{ij}$ in the adjacency matrix $A= [a_{ij}]$ takes value $1$ if there is a directed connection from node $j$ to node $i$ and $0$ otherwise. Hence the dynamics in \eqref{eq:nn-individual-node-local} can be equivalently written as 
\begin{equation}\label{eq:nn-individual-node}
	s_{i}(k+1) = \sum_{j=1}^n a_{ij} \Psi( w_{ij},s_j(k)), \quad i \in [n], ~k \in \BN_0 ,
\end{equation}
 where $s_i(k) \in [0,+\infty]$ for all $i\in [n]$ and all $k\in \BN_0$.

\subsection{Sufficient Condition for Virus Extinction}
 Let  $A=[a_{ij}]$ and $W=[w_{ij}]$.  Let $\odot$ denote the Hadamard product, and $\{\lambda_{i}(A\odot W)| i\in[n]\}$ denote all the eigenvalues of $A\odot W$.  
\begin{theorem}[Sufficient Condition for Virus Extinction]\label{eq:threshold}
	The virus spread characterized by \eqref{eq:probablistic-dynamics} and equivalently by \eqref{eq:nn-individual-node} will die out regardless of initial conditions if  the eigenvalues of $A\odot W$ are less than $1$ in absolute values, i.e., 
	\begin{equation}\label{eq:exp-stable}
	\max_{i\in[n]}|\lambda_{i}(A\odot W)| <1.
 \end{equation}
\end{theorem}
See Appendix \ref{sec:threshold-proof} for the proof.
\begin{remark}[Special Case]
Specializing $W$ to the following
\[
w_{ii} =1-\delta, \quad w_{ij} = \beta, \quad \forall i, j \in [n], \text{~with~} \delta, \beta \in (0,1)
\]
yields the virus spread model in \cite{chakrabarti2008epidemic}. 
In this case  
$
A\odot W  = \beta A + I (1-\delta-\beta)
$ is symmetric with non-negative elements.
Then the inequality condition \eqref{eq:exp-stable}  reduces to 
$
\beta \lambda_{\max}(A)+1-\delta - \beta <1
$ which  is equivalent to 
$
 \lambda_{\max}(A) < \frac{\delta+\beta}{\beta} = \frac{\delta}{\beta} + 1
$ where $\lambda_{\max}(\cdot)$ denotes the largest eigenvalue. This condition is equivalent to  the condition $
 \lambda_{\max}(\tilde{A}) <\frac{\delta}{\beta}
$ established in \cite{chakrabarti2008epidemic},  
where $\tilde{A}\triangleq A -I$ is the adjacency matrix of the underlying physical contact network excluding self-loops. %
\end{remark}

\section{Probabilistic Transmission Networks: Multiple Particle Transmissions}\label{sec:multi-probabilistic-trans}

\subsection{Multiple Particle Transmission Model}

 We consider the following virus spread model with multiple particle transmissions at each link
\begin{equation} \label{eq:population-model}
	1-  p_h(k+1) = \prod_{q\in {N}_h^\circ} \Big(1- w_{_{hq}} p_q(k)\Big)^{a_{_{hq}}},~ h\in [n]
\end{equation}
where  $p_q(k)$ denotes the probability of infection at node $q$ at time $k$, $w_{_{hq}} \in[0,1]$ is the probability that each particle (e.g. virus) transmitted from node $q$ to node $h$ causes an infection at node $h$, 
and $a_{_{hq}}$ is  the number of particles (e.g. viruses) transmitted from node $q$ to node $h$, and $N_h^\circ \triangleq \{q: (h,q)\in E \}$ denotes the neighbourhood of node $h$ on the physical contact network as defined in Section \ref{sec:single-probablistic-transmission}.

We present three interpretation examples for this multiple particle transmission model: 
 (a) Micro-level epidemic spread over contact networks where multiple viruses (or droplets) are sent across a physical contact link; 
 (b) Population-level epidemic spread over transportation networks among cities (or countries) where multiple infected individuals commute over each transportation link;
 (c) Spread of neuronal excitations on neuronal networks with chemical synapses,  where multiple neurotransmitters are released from presynaptic neuron to synaptic cleft to bind to receptors at the postsynaptic neuron for each chemical synapse.

\subsubsection{Micro-Level Epidemic Spread} 
{On a micro scale, the spread of a disease may depend on  the transmission of multiple virus particles  among nodes (e.g. persons)}.
 Each infected node (e.g. person) contains a population of virus particles. At a physical contact link, a proportion of the virus population spreads from one node to another. {It is assumed that at each physical contact link, the transmission of each virus particle  is independent of those of other virus particles, and can independently infect the receiving node with the same probability.} At each contact link, there can be a large number of independent virus transmissions. Then $a_{_{hq}}$ denotes the number of independent virus transmissions from node $q$ to node $h$.

\subsubsection{Epidemics Spread among Cities}
Consider a network of cities (or countries) connected via a transportation network. Each node represents a city (or a country) that has a population. 
The \emph{effective transportation flow} $a_{_{hq}}$ is defined as the number of the infected that still travel from the source node $q$ to the target node $h$.   Then $p_q(k)$ denotes the probability state of infection at node $q$ at time $k$, and  $w_{_{hq}}$ is the probability of infecting  node $h$  by  a single infected person from node $q$.  
In practices,  the underlying effective transportation flow $[a_{_{hq}}]$ may be obtained from virus testings for passengers or approximately estimated based on passenger flows\footnote{If the detailed information regarding virus testing for passengers is not available, then we may tune a certain parameter $\rho_q \in [0,1]$ representing a proportion of the infection in the passenger flow such that the effective transportation flow is given by  $a_{_{hq}} = \rho_q T_{hq}$, where $T_{hq}$ denotes the passenger flow from city $q$ to city $h$. }.  

\subsubsection{Neuronal Networks with Chemical Synapses}
	Consider a network of biological neurons connected over  chemical synapses.  
	An illustration of a typical chemical synapse is shown in Fig.~\ref{fig:chemical-synapse}. Let's restrict our attention to the case where all neuronal connections are excitatory. 
	An excited neuron (as the presynaptic neuron) releases multiple neurotransmitters into the  synaptic cleft\footnote{The connection between two neurons may involve multiple chemical synapses. In terms of mathematical modelling, we can treat all of the synapses as one synaptic connection.}.  Some of the released neurotransmitters will bind to receptors in the postsynaptic neuron, which may lead to  the excitation of the postsynaptic neuron. For further details on the release mechanisms of neurotransmitters, their binding to receptors and their role in activating postsynaptic neurons, readers are referred to \cite[Part III]{kandel2000principles}. %
	Consider a neuron $q$ releasing $a_{_{hq}}$ number of neurotransmitters into the synaptic cleft between neuron $q$ and neuron $h$, and {assume that each neurotransmitter  in the synaptic cleft can randomly bind to a receptor with probability $w_{_{hq}}$\footnote{As excitations of neurons may be caused by different types of neurotransmitters that may differ among connections, it is suitable to have a parameter $w_{_{hq}}$ to represent differences in transmission probabilities potentially caused by different types of neurotransmitters and their receptors.} independently\footnote{Whether this is a good assumption may depend on the number of receptors and the spatial configurations of synapses. 
} from other neurotransmitters}. Then the model in \eqref{eq:population-model}  characterizes  the probabilities of neuronal excitations  on a synaptic network at a future time given the current probabilities of neuronal excitations.

\begin{figure}[htb]
\centering
	\includegraphics[width=8cm]{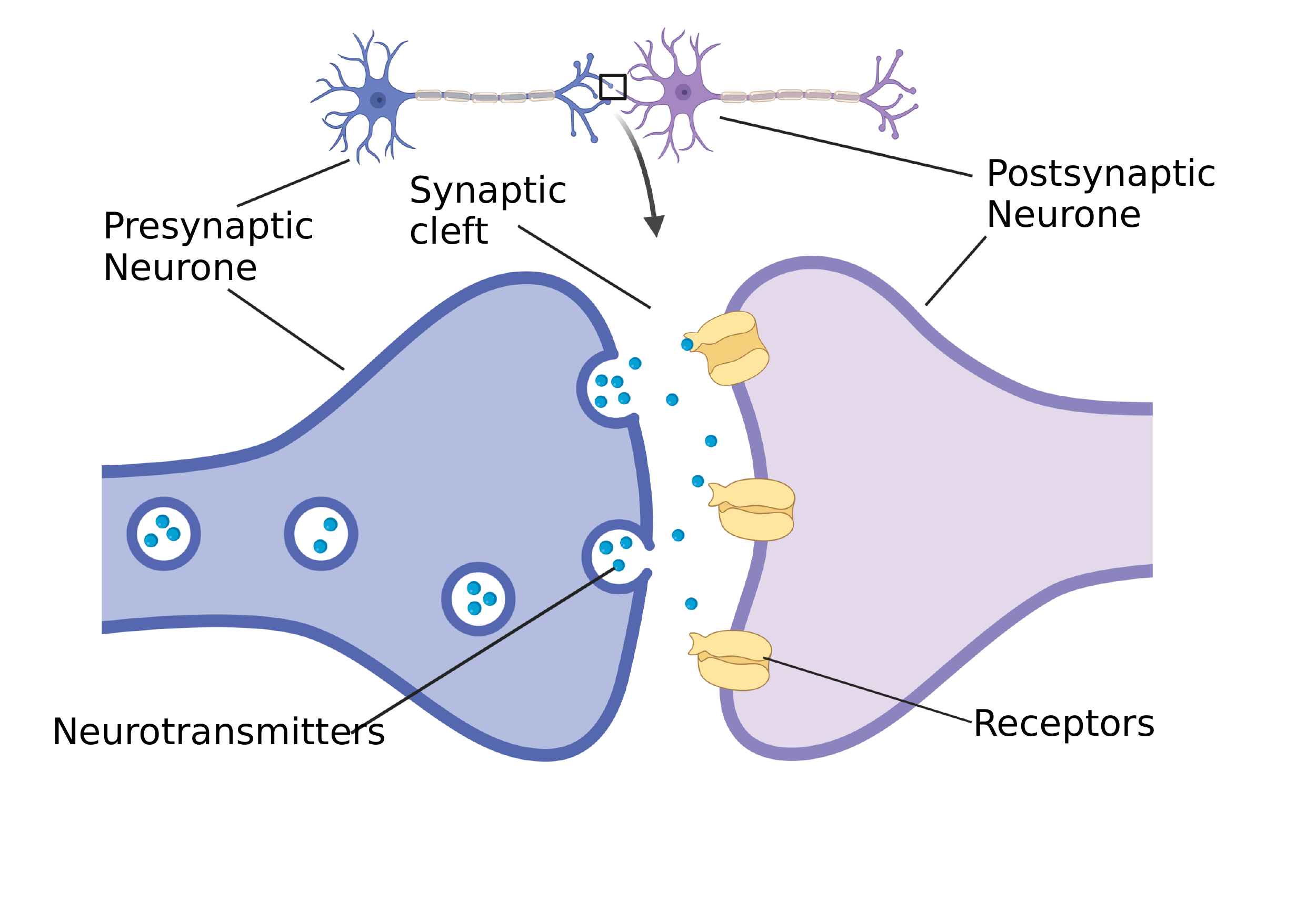}\\
	\caption{A schematic illustration of a typical chemical synapse \footnotesize{(created with BioRender.com}). } \label{fig:chemical-synapse}
\end{figure}
\begin{remark}[Integer Weights and Rational Weights]
{We note that the weight $a_{_{hq}}$ presented above is a non-negative integer. Some  relaxations are required to allow $a_{_{hq}}$ taking {rational values} to associate it to artificial neural network models with rational weights\footnote{In training artificial neural networks  in practice, weights are often rational numbers and irrational numbers are rarely used.}.  One interpretations of rational weights may be  through  \emph{neuronal twins}\footnote{Neuronal twins here are defined as neurons that are connected to the same set of other neurons within the synaptic network under consideration.}.  For example, $(\cdot )^{\frac{10}{3}}$ means that 10 transmitted particles would lead to the excitation of 1 out of the 3 postsynaptic neuronal twins. This seems to suggest that there should be a population of neuronal twins  associated to each node to enable the fractional power. By definition,  neuronal twins do not need to be  near each other physically and anatomically, but they must have the same connections.} 
Based on this interpretation, biological neuronal networks may seem to have at least two ways to adjust the rational weight $a_{hq}$: (i) the adjustment of the number of neurotransmitters at each connection and (ii) the adjustment of the number of twin neurons  that share the same connections at each node.  
\end{remark}

Salient features of the model in \eqref{eq:population-model} are  that (a) multiple particles are transmitted from one node to another and (b) at each link the successful transmissions of particles are assumed to be independent from those of other particles.  The particles can be interpreted across different scales which could represent virus particles, or infected travellers in epidemic networks, or neurotransmitters in neuronal networks.
\subsection{Equivalent Representations via Neural Networks}
Taking $\log$ on both sides of \eqref{eq:population-model}, we rewrite the  model as 
\begin{equation*}
	\begin{aligned}
	\log(1- p_h(k+1)) &=  \log\prod_{q\in  {N}_h^\circ} \Big(1- w_{_{hq}} p_q(k)\Big)^{a_{_{hq}}} \\
	 =  \sum_{q\in  {N}_h^\circ}&a_{_{hq}}\log\Big(1- w_{_{hq}} p_q(k)\Big), 
	~~ \forall h, q \in [n],
\end{aligned}
\end{equation*}
where ${N}_q^\circ$ denotes the set of neighborhood nodes of node $q$ (including itself) on the physical contact network.
Let each individual node  maintains two types of states:
\begin{equation*}
	\begin{aligned}
		& \text{Input State}:\quad~~~ s_q({k})\triangleq - \log\Big(1- p_q(k)\Big), \\
		&\text{Output State}: \quad ~o_q^{h}({k})\triangleq -\log\Big(1- w_{_{hq}} p_q(k)\Big),
	\end{aligned}
\end{equation*}
where $k\in \BN_0$ and $h,q \in [n]$. 
{We note that the output state depends on the receiving neurons as well and hence strictly speaking it is not a state of a node but a state of a connection.}
Then the relation between input state and output state is
\[
\begin{aligned}
o_q^{h}(k) = \Psi(w_{_{hq}}, s_q(k)) 
	& = - \log\left(1 - w_{_{hq}} +w_{_{hq}} e^{-s_q(k)}\right)
\end{aligned}
\]
where  $\Psi(\cdot, \cdot)$ above is the TLogSigmoid activation function defined in \eqref{eq:typeI-Phi}. 
Then the dynamics in \eqref{eq:population-model} are equivalent to 
\begin{equation}\label{eq:pop-nn}
	s_h({k+1}) = \sum_{h=1}^na_{_{hq}} \Psi(w_{_{hq}}, s_q(k)), 
	~~ \forall h, q \in [n],
\end{equation}
where  $s_q(k) \in [0,+\infty]$ for all $q\in [n]$ and $k\in \BN_0$, and $a_{_{hq}}$ is the network weight representing the number of particles transmitted from $q$ to $h$ (and clearly  $a_{_{hq}}=0$ if node $h$ and $q$ are not neighbors on the physical contact network).  We call this model the \emph{Transmission Neural Network (TransNN) with multiple particle transmissions}.  
Following the same proof of Theorem \ref{eq:threshold}, we obtain that  $\lim_{k\to \infty} p(k) =0$ and $\lim_{k\to \infty} s(k) =0$ regardless of initial conditions if  $\max_{i\in[n]}|\lambda_{i}(A\odot W)| <1$.

We note that if $w_{_{hq}}=1$ then $o_q^{h}(k) = s_q(k)$ and if $w_{_{hq}} =0$ then $o_q^{h}(k)=0$. 
The activation level of $\Psi(w, \cdot)$ is governed by $w$ as illustrated by Fig.~\ref{fig:activation-psi1}.  The parameter $w$ characterizes the activation levels of individual neural connections, from no pass (with parameter $0$), to partial pass (with parameters in $(0,1)$) and  to full pass (with parameter $1$). See Fig.~\ref{fig:activation-psi1}.  %

{In biological synaptic networks, inhibitory neurons (external to the nodes of the synaptic networks under consideration) and inhibitory neurotransmitters may alter $w_{_{hq}}$ (i.e. the inhibition level of neural link from $q$ to $h$).}
Furthermore, the change of $w_{_{hq}}$ for a group of neurons may potentially be used to model the function of neuromodulators that alters  effective synaptic strengths \cite{marder2012neuromodulation}.

\subsection{Modulating Activation Levels}
Activation levels for links in TransNNs may be modified or modulated globally or nodally.%
\subsubsection{Global Modulation}
A global influence $\gamma \in [0,1]$ can be introduced to simultaneously modulate the activation  level of all nodes in TransNNs,  that is, 
\[
w_{_{hq}} = \gamma c_{_{hq}}, \quad \forall  h, q \in [n],
\]
where $c_{_{hq}} \in [0,1]$, $ h, q \in [n]$, are some fixed probabilities inherent to the TransNN system.

\subsubsection{Dual Nodal Modulation}
In chemical synapses, both presynaptic and postsynaptic neurons can be modulated  to change  the probability of neurotransmitters binding to receptors or the amount of released neurotransmitters. %
 Motivated by these, we introduce the presynaptic modulation $\beta_q \in [0,1]$ and postsynaptic modulation $\alpha_h \in  [0,1]$ in TransNNs as extra variables that modify the effective transmission probability as follows:
\[
w_{_{hq}} = \alpha_h c_{_{hq}} \beta_q, \quad \forall  h, q \in [n], 
\]
where $c_{_{hq}} \in [0,1]$, $ h, q \in [n]$, are some fixed probabilities inherent to the TransNN system.

\section{The TLogSigmoid Activation Function and Its Variants}\label{sec:Psi-Activation}
Recall the definition of the  TLogSigmoid activation function denoted by $\Psi(\cdot, \cdot)$ in \eqref{eq:typeI-Phi} as
\begin{equation}
	\Psi(w, x) = -\log \left(1- w + we^{-x} \right),
\end{equation}
where $w \in [0,1]$ and $x\in [-\infty, +\infty]$. 
 The shape of the TLogSigmoid activation function  is illustrated in Fig.~\ref{fig:activation-psi1}. 
When the tunable parameter $w=0.5$, the TLogSigmoid activation function denoted by $\Psi(w,x)$ is related to the sigmoid function $\sigma(x)= (1+e^{-x})^{-1}$ as follows: for $x\in \BR$,
\[ \sigma(x) = 0.5 e^{\Psi(0.5, x)}
 ~~\text{and}~~ \Psi(0.5,x) = \log(2\sigma(x)).
\]
Log-Sigmoid type functions have been used in the literature and in practice:  the  function $2\log(2\sigma(x)) =2\Psi(0.5,x)$ has been used in constrained optimization in \cite{polyak2001log}; the function $\log(\sigma (x))$ has been recently implemented as an activation function in standard machine learning packages (e.g. PyTorch and TensorFlow).  Salient features of the TLogSigmoid function that distinct itself from these existing Log-Sigmoid type functions are  the tunable parameter $w$, and non-negative parts of the TLogSigmoid function. 
\subsection{Derivatives of TLogSigmoid $\Psi(w,x)$ with Respect to $w$}
The partial derivative of $\Psi(w, x)$ with respect to $w \in (0,1)$ is given by
\begin{equation} \label{eq:psi-1stgradient}
	\partial_w{\Psi({w}, x)}= \frac{1- e^{-x}}{1- w +w e^{-x} } = (1-e^{-x}){e^{\Psi(w,x)}}.
\end{equation}
{The denominator $1+w(e^{-x} -1)$ is non-zero for any $x\in[-\infty, +\infty)$ and any ${w\in (0,1)}$.   The derivative above is $+\infty$ when $x=+\infty$ and $w\neq 0$. Moreover
%
%
it  
has interesting properties: 
for $w\in (0,1)$, 
\[
\begin{aligned}
	& \partial_w{\Psi({w}, x)}=0,\quad \text{ when $x=0$};\\
	& \partial_w{\Psi({w}, x)}<0,\quad \text{ when $x< 0$;}\\
	& \partial_w{\Psi({w}, x)}>0,\quad \text{ when $x> 0$}.
\end{aligned}
\]
The partial derivative of $\partial_w\Psi(w, x)$ can be conveniently integrated into  gradient backpropagation or automatic differentiation \cite{bucker2006automatic} when TransNN models \eqref{eq:nn-individual-node} and \eqref{eq:pop-nn} (after a slight generalization) are trained as function approximators. 
Interestingly, the partial derivative of $\Psi(\cdot, \cdot)$ with respect to $w$ is related to the tanh function $\tanh(x) = \frac{e^{x}-e^{-x}}{e^{x}+e^{-x}}$  as follows: 
\begin{equation*}
\begin{aligned}
	& \partial_w \Psi(0.5, x ) = 2\tanh (0.5 x),\quad \forall x\in [-\infty, +\infty].
\end{aligned}
\end{equation*}

Higher order partial derivatives of $\Psi(w,x)$ with respect to $w$ are  explicitly given as follows: for $w\in (0,1)$ and $k\geq 1$,
\[
\begin{aligned}
	{\partial^k_w }{\Psi({w}, x)} & = (k-1)!{(1-e^x)^k}{e^{k\Psi(w, x)}},   \quad  x\in [-\infty, +\infty]. 
\end{aligned}
\]
We note that ${\partial^k_w }{\Psi({w}, 0)}=0$ for any $w\in [0,1]$. Furthermore,  for any $x\in \BR$, $\Psi(w,x)$ is convex in $w \in[0,1]$, since  ${\partial^2_w }{\Psi({w}, x)}$  is always non-negative. 

\subsection{Derivatives of TLogSigmoid $\Psi(w,x)$ with Respect to $x$}
The partial derivative of $\Psi(w,x)$ with respect to $x \in [-\infty, +\infty]$ satisfies
\begin{equation}\label{eq:psi-grad-x}
	\partial_x \Psi(w, x) = \frac{w e^{-x}}{1 - w +w e^{-x}} = {w e^{-x}}{e^{\Psi(w, x)}} \geq 0
\end{equation}
for any $w\in[0,1]$.
If $x\in [0,+\infty]$, then
$\partial_x \Psi(w, x)$ is monotonically increasing in $w\in[0,1]$. Thus,  setting a small $w$ makes $\Phi(w, x)$  less sensitive to variations of the signal input $x\in [0,+\infty]$, which, in other words,  increases the robustness of $\Phi(w, x)$ with respect to $x$.
We note that $\partial_x \Psi(w, x)$ is related to  the sigmoid function $\sigma(x)=({1+e^{-x}})^{-1}$  as follows: %
\begin{equation*}
\begin{aligned}
	&\partial_x \Psi (0.5, - x) = \sigma(x), \quad \forall x\in [-\infty, +\infty].\\
\end{aligned}
\end{equation*}
{The 2nd order partial derivative of $\Psi(w,x)$ with respect to $x$ satisfies: for any $w\in[0,1]$ and any $x\in [-\infty, +\infty]$,
\begin{equation}\label{eq:2nd-Psi-grad-w}
	{\partial^2_x }\Psi(w,x) = (-1) \partial_x \Psi(w, x)(1- \partial_x \Psi(w, x))\leq 0. 
\end{equation}
Thus for any $w\in[0,1]$, $\Psi(w,x)$ is concave in $x \in [-\infty, +\infty]$. 
More generally, higher order partial derivatives of $\Psi(w,x)$ with respect to $x \in[-\infty, +\infty]$ are explicitly given as follows: for  $ n \geq 2$ and any $w\in[0,1]$,
\begin{equation}\label{eq:high-grad-psi-form}
		{\partial^n_x } \Psi(w,x)   = \sum_{k=1}^{n} (-1)^{^{k+n}} (k-1)! S_{n, k}(\partial_x \Psi(w, {x}))^k
\end{equation}
where $S_{n,k}$  denotes the Stirling numbers of the second kind (see e.g. \cite[Chapter 6.1]{graham1989concrete}) 
and
	$
	\partial_x \Psi(w, x)  
$ is given by \eqref{eq:psi-grad-x}. See Appendix \ref{sec:high-grad-phi} for the derivation of \eqref{eq:high-grad-psi-form}.}

\subsection{Variants of TLogSigmoid}
\subsubsection{Tunable Log-Sigmoid-Plus Activation  $\Psi_{+}(w, x)$}
In the virus spread model,  the input $x$ of $\Psi(w,x)$ as the negative-log-negative probability state is always non-negative. Thus the activation function used is essentially the following: 
\begin{equation}\label{eq:PsiPlus}
\Psi_{+}(w, x)  \triangleq \begin{cases}
	  - \log \left(1- w + we^{-x} \right),& x\geq 0 \\
	 0, & x<0
\end{cases}
\end{equation}
for $w\in [0,1]$. 
\begin{figure}[htb]
\centering 
\includegraphics[width=8cm,trim={3.8cm 0 3.2cm 0},clip]{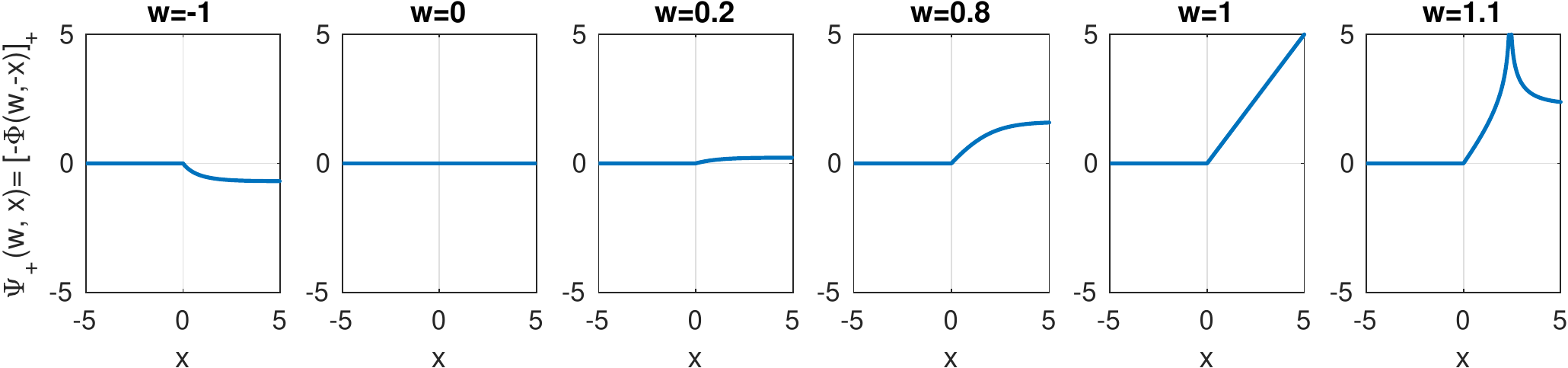}
	\caption{TLogSigmoidPlus activation function $\Psi_{+}(w,x)$ with different $w$  in $[0,1]$. When $w=0$, $\Psi_{+}(w,\cdot)$ is zero; when $w=1$, $\Psi_{+}(w,\cdot)$ becomes ReLU; when $w \in(0,1)$, $\Psi_{+}(w,\cdot)$ is non-polynomial.}\label{fig:activation-psi-plus1}
\end{figure}
See Fig. \ref{fig:activation-psi-plus1} for an illustration of $\Psi_{+}(w, x) $.  In particular, we notice that it becomes ReLU activation when $w=1$. The activation function $\Psi_{+}$ in \eqref{eq:PsiPlus} has only non-negative outputs and hence we call it the tunable Log-Sigmoid-Plus  (TLogSigmoidPlus) function.

\subsubsection{Tunable Soft-Affine Activation $\Phi(w,x)$}
If we take the state to be $\bar{s}_i = \log(1-p_i)$, $i\in V$,  then virus spread dynamics  give rise to  the essentially same TransNN models in \eqref{eq:nn-individual-node} and \eqref{eq:pop-nn} with the new states $(\bar{s}_i)_{i\in V}$ (see Appendix \ref{sec:Phi-activation}),  but with a different activation function given as follows: %
\begin{equation}\label{eq:phi-activation}
	\Phi(w, x) \triangleq \log(1-w+we^x) = -\Psi(w, -x),
\end{equation}
where  $x\in[-\infty, +\infty]$ and  $w\in [0,1]$. 
This activation function denoted by $\Phi(w,x)$ can be considered as a Softplus activation with a tunable parameter $w$ and potentially negative parts (as illustrated in Fig. \ref{fig:activation-phi}). 
	 {For this reason, we call this activation function denoted by $\Phi(w,x)$ the tunable Soft-Affine  (TSoftAffine) function}. See Appendix \ref{sec:Phi-activation} for more properties on the TSoftAffine activation function and its derivatives.
\begin{figure}[htb]
\centering 
\includegraphics[width=8cm, trim={3.8cm 0 3.2cm 0},clip]{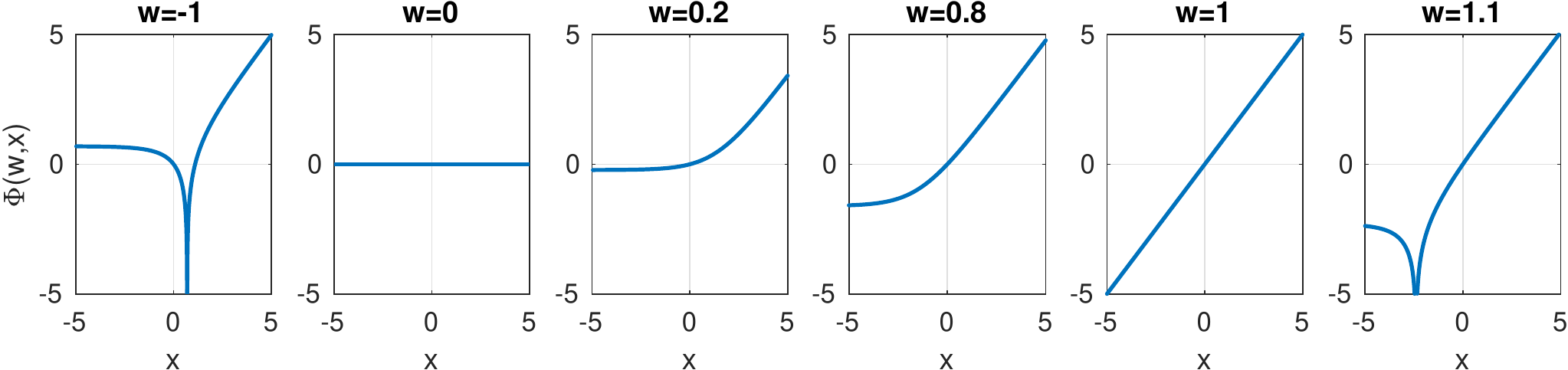}
	\caption{TSoftAffine activation function $\Phi(w,\cdot)$ with $w$ taking different values from $0$ to $1$. When $w=0$, $\Phi(w,\cdot)$ is zero and represents no pass; when $w=1$, $\Phi(w,\cdot)$ is linear and represent full pass; when $w\in(0,1)$, $\Phi(w,\cdot)$ is non-polynomial, monotonically increasing and continuous with negative values on the left plane, and represents partial pass. } \label{fig:activation-phi}
\end{figure}
\subsection{Experiments with Different Activation Functions}\label{sec:experiment}
In experiments we compare performances  of different activation functions on  the same simple neural network structure illustrated in  Fig. \ref{fig:nn-structure}.  The activation functions used include non-tunable activation functions (such as ReLU, sigmoid, tanh, SiLU, soft-exponential,  TLogSigmoid $\Psi$-activations and TSoftAffine $\Phi$-activations with fixed parameters), and  trainable activation functions  (such as the TLogSigmoid $\Psi$-activation and TSoftAffine $\Phi$-activation).

 \begin{figure}
 \centering
 	\includegraphics[width=8cm]{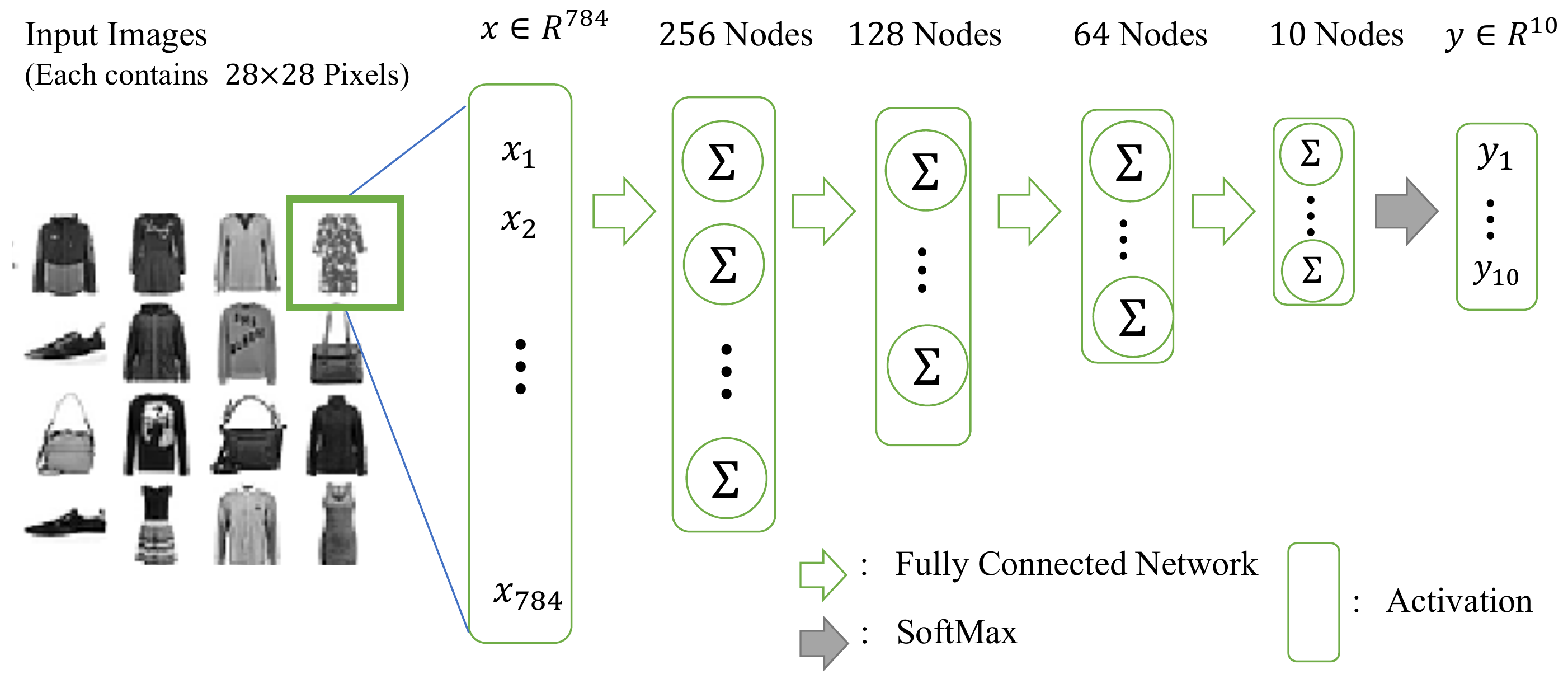}
 	\caption{Neural network structure for experiments in Sec. \ref{sec:experiment}} \label{fig:nn-structure}
 \end{figure}

The neural network structure consists of 4 consecutive fully-connected hidden layers  with respective dimensions  256, 128,  64 and 10. The output layer has 10 nodes and the  prediction output is the logarithm of the softmax of the nodal values in the last hidden layer. The training and validation criteria are chosen to be the negative-log-likelihood loss.  ADAM is used as the optimizer for gradient updates. The average training and validation errors over the number of epochs (i.e. the number of complete passes of the training dataset) of FashionMNIST data \cite{xiao2017fashion} are illustrated in Fig.~\ref{fig:train-val-diff-activations}.

\begin{figure}[htb]
\centering
	\includegraphics[width=8cm]{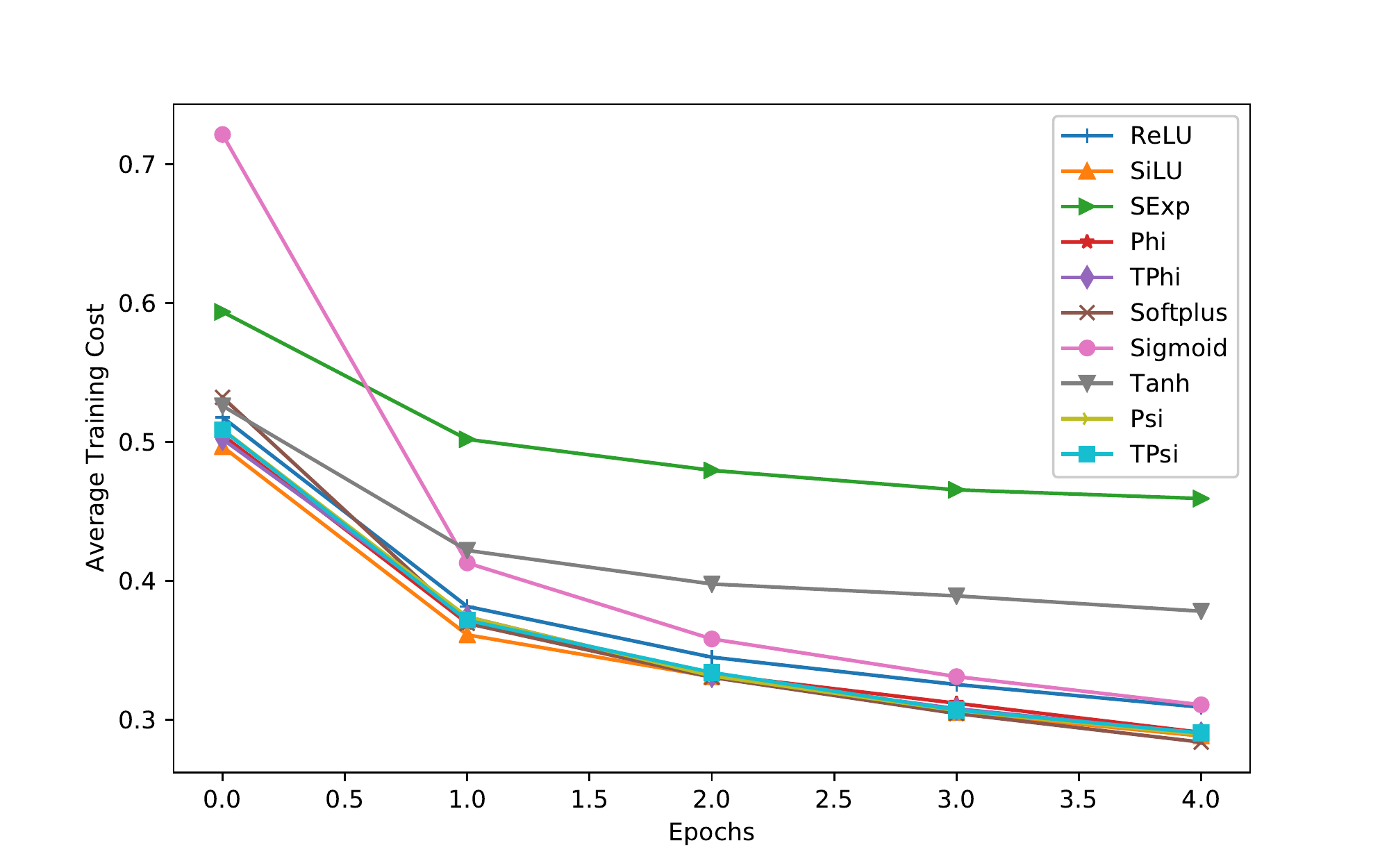}
		\includegraphics[width=8cm]{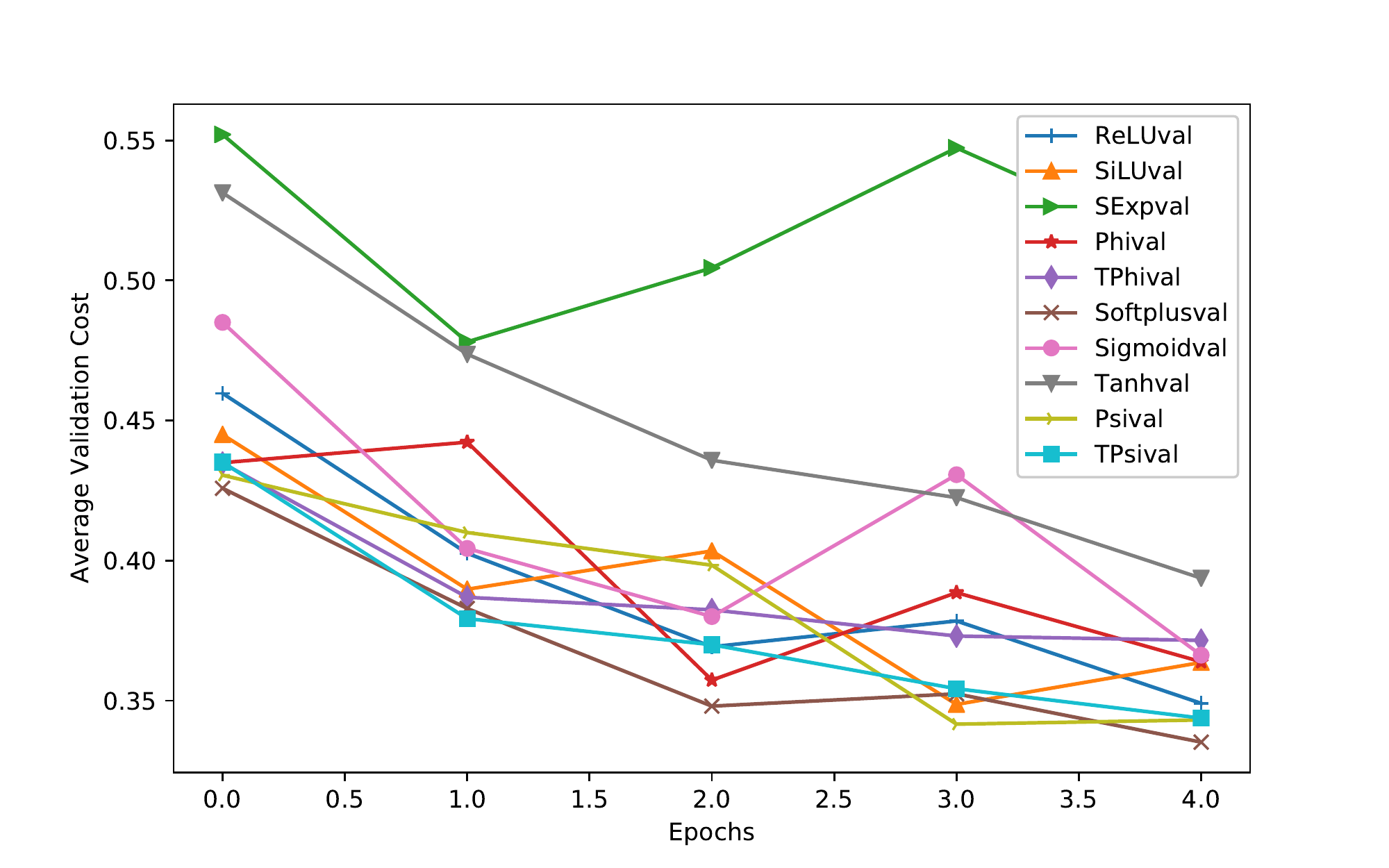}
	\caption{Training and validation costs of the  neural network structure in Fig. \ref{fig:nn-structure} with different activation functions using FashionMNIST data \cite{xiao2017fashion}. TPhi and TPsi denote respectively  $\Phi$ and $\Psi$ activations with trainable parameters. Phi and Psi denote respectively $\Phi$ and $\Psi$ with fixed  parameters.}
	\label{fig:train-val-diff-activations}
\end{figure}

\section{TransNNs in General Forms} \label{sec:TransNNs}

\subsection{Time-Varying and Layer-Dependent Networks}
In previous sections,  it is implicitly assumed the underlying networks do not change over time. 
To incorporate the time-varying (or layer-dependent) network weights and structures and time-varying  (or layer-dependent) activations, we can naturally generalize the dynamics as follows:
\begin{equation}\label{eq:TransNNGeneral}
\textbf{TransNN:} \quad 	s_{i}(k+1) = \sum_{j=1}^n a_{ij}^k \Psi( w_{ij}^k,s_j(k)), ~ i \in [n], 
\end{equation}
where $k \in \{0,..., T-1\}$, $a_{ij}^k \in \BR$, $w_{ij}^k\in [0,1]$, and 
$$ 
\Psi( w^k_{ij},s_j(k)) = - \log \left(1- w_{ij}^k + w_{ij}^ke^{- s_j(k)} \right).
$$ %
Clearly in the model above each time step can be associated to a single layer of neurons. The activation function TLogSigmoid denoted by $\Psi$ here can be replaced by TLogSigmoidPlus denoted by $\Psi_{+}$ or TSoftAffine denoted by $\Phi$.

In the most general form of TransNNs,  $a_{ij}^k$ and $s_i(k)$ are allowed to take any real values, and $w_{ij}^k$ takes values in $[0,1]$. 

{When the weight $a_{ij}^k$ is a non-negative rational and the nodal state $s_i(k)$ is non-negative for all $i,j \in [n]$, the model in \eqref{eq:TransNNGeneral} with TLogSigmoid activations can be exactly associated with virus spread models with (non-negative) probability states.}
In the corresponding virus spread model with probability states, the probability of node $i$ being infected at time $k+1$ satisfies 
\begin{equation} \label{eq:tv-probablistic-dynamics}
	(1-p_i(k+1)) = \prod_{j\in N_i^{\circ k}} (1- w_{ij}^k p_j(k))^{a_{ij}^k}
\end{equation}
where $N_i^{\circ k}$ denotes the set of neighbors of node $i$ including itself  at time $k \in \{0,..., T-1\}$. 

\begin{remark}[Negative Probability]
We note that $s_i(k) = -\log(1-p_i(k))$ is always non-negative in the virus spread model since $1-p_i(k) $ as a (non-negative) probability lies in $[0,1]$. In the general form of TransNNs as neural network models in \eqref{eq:TransNNGeneral}, $s_i(k)$ is allowed to take essentially any real value (including negative values), which means the probability of infection $p_i(k)$ may be negative as an intermediate variable.  The notion of negative probability has been discussed in \cite{feynman1987negative} by Feynman:   ``Conditional probabilities and probabilities of intermediate states may be negative in a calculation of probabilities of physical events or state".  An example presented in \cite{feynman1987negative} is that  the probability of a diffusing particle being at a location for an eigendirection could be negative. %
\end{remark}

\subsection{TransNNs as Learning Models}
{When we consider the equation \eqref{eq:TransNNGeneral} above as a neural network with feedforward connections, %
the neural network input is $s(0)\triangleq [s_1(0),..., s_n(0)]^\TRANS$ and the neural network output is $s(T)\triangleq [s_1(T),..., s_n(T)]^\TRANS$. 
That is 
\[
s(T) = \text{TransNN}_\theta(s(0)),
\]
where $\theta \triangleq (n, T, [a_{ij}^k],[w_{ij}^k])$, and  $[a_{ij}^k]$ (resp. $[w_{ij}^k]$) denotes the tensor containing elements $a_{ij}^k$ (resp. $w_{ij}^k$) with $i,j\in [n]$ and $k\in \{0,...,T\}$.
Given  the set of $D$ input-output data pairs $\{s^{(i)}(0), y^{(i)}\}_{i=1}^D$, 
the objective of training is to identify the parameters in TransNNs that minimize certain cost, for instance,  given by
\[
\min_{\theta \in \Theta} \left\{ \frac1D \sum_{i=1}^D l\left(\text{TransNN}_\theta(s^{(i)}(0)), ~y^{(i)}\right) +r(\theta)\right\},
\]
where $l(\cdot, \cdot)$ is a loss function or a distance function, and $r(\theta)$ represents the regularization cost of the parameter $\theta \in \Theta$, and $\Theta$ is the set of all feasible parameters.
We note that in \eqref{eq:TransNNGeneral}, the numbers of nodes may seem the same across layers. However, if we would like to let different layers have different number of nodes, we can simply use the maximum number of nodes across all layers  as the number of nodes for each layer, and then assign unactivated nodes (i.e. nodes without input connections) to  layers with less number of nodes.
 
\begin{remark}
The final output of TransNNs may take other forms. For example, 
the observation output for the nodes may be the probabilities  $p(T)\triangleq [p_1(T),..., p_n(T)]^\TRANS$ specified by \eqref{eq:neg-log-neg} as
$
 s_i = - \log(1-p_i)  \text{ and }   p_i  = 1 - e^{-s_i};
$
that is, $p$ is the nonlinear observation of the state $s$ as
\[
p = 1-\exp_\circ(-s) \triangleq o(s).
\]
The associated objective may be specified by x
\[
\min_{\theta \in \Theta} \left\{ \frac1D \sum_{i=1}^D l\left(o(\text{TransNN}_\theta(s^{(i)}(0))), ~y^{(i)}\right) +r(\theta)\right\}.
\]
Another example is that the output can also be considered as a function of the observation sequence of nodal states
\[
\hat{y}_\theta = f(\text{TransNN}^1_\theta(s(0)), \cdots , \text{TransNN}^T_\theta(s(0)) )  %
\]
with $\text{TransNN}^k_\theta(s(0)) \triangleq s(k)$.
Then the associated learning objective may  be specified for instance by
\[
\min_{\theta \in \Theta} \left\{ \frac1D \sum_{i=1}^D l\left(\hat{y}_\theta(s^{(i)}(0)), ~y^{(i)}\right) +r(\theta)\right\}. 
\]
\end{remark}
  The training of TransNNs can be conveniently carried out via gradient descend with automatic differentiation.  %

\begin{remark}[Automatic Selection of Activation Functions]
	During training of the TransNNs, treating the activation-level parameters $w_{**}$ as trainable parameters essentially provides the flexibility of automatically selecting activation functions from a continuum class of activations, which in addition includes linear, sigmoid, tanh, Softplus, ReLU, and LogSigmoid, since all these activation function are  special cases of TLogSigmoid, TLogSigmoidPlus and TSoftAffine, or the derivatives of them.
\end{remark}

\section{Multilayer TransNNs are Universal Function Approximators} \label{sec:universal-approx-TransNNs}
We follow the density-type definition of universal function approximators with arbitrary width in   \cite{pinkus1999approximation,leshno1993multilayer,hornik1989multilayer}. For a set $K$, let $C(K)$ denote the set of continuous functions from $K$ to $\BR$ and $C(K;\BR^m)$ the set of continuous functions from $K$ to $\BR^m$.   
A function $u:\Omega \to \BR^m$ defined almost every on a domain $\Omega$ 
is said to be \emph{locally essentially bounded} on $\Omega$, denoted by $u\in L^\infty_{\operatorname{loc}}(\Omega;\BR^m)$, if for every compact set $K\in \Omega$, $\operatorname{ess}\sup_{x\in K} \|u(x)\|< \infty$, where $\|\cdot\|$ denotes the Euclidean norm in $\BR^m$. 
%
A set $\mathcal{M}$ of (parameterized) functions in $L^\infty_{loc}(\BR^d;\BR^m)$ is called a \emph{Universal Function Approximator for $C(\BR^d;\BR^m)$} if 
 given any 
${\varepsilon >0}$, any compact subset of ${K\subseteq \BR^{d}}$ and any 
${f\in C(K)}$, there exists 
${F\in {\mathcal {M}}}$ such that
$$
{\operatorname{ess}\sup_{x\in K}\|F(x)-f(x)\|<\varepsilon },
$$
{or equivalently 
${\|F(x)-f(x)\|<\varepsilon }$ for almost all $x\in K$}.
In other words, 
$
\mathcal{M}
$ is a universal function approximator for $C(\BR^d;\BR^m)$ if it is {\emph{dense}}  
in $C(\BR^d;\BR^m)$  in the topology of uniform convergence on compacta (\cite{pinkus1999approximation,leshno1993multilayer,hornik1989multilayer}). %
\begin{proposition} \label{prop:TransNN-typeII}
Let $w\in(0,1)$ be given. Consider the parameter set  
\begin{equation}
 	\begin{aligned}
 		 \Theta_0 \triangleq \Big\{(n,  (a_{i})_{i=1}^n, &(\eta_i)_{i=1}^n, (b_i)_{i=1}^n)\big|  \\
 		 & n\in \BN, a_i, b_i \in \BR, \eta_i\in \BR^n \Big\}.
 	\end{aligned}
 \end{equation} Then feedforward neural network model  with one hidden layer from $\BR^d \to \BR$ given by
\begin{equation} \label{eq:ffphi}
y^\theta(x) = \sum_{i=1}^n a_{i} \Psi(w, \eta^\TRANS_i x+ b_i) , \quad x\in \BR^d, ~y^\theta(x) \in \BR,
\end{equation}
with arbitary parameters $\theta  \triangleq (n,  (a_{i})_{i=1}^n, (\eta_i)_{i=1}^n, (b_i)_{i=1}^n)$  in $\Theta_0$,  is a universal function approximator for $C(\BR^d)$. 

\end{proposition}
\begin{proof}
 We observe  that when  $w \in (0, 1)$, $\Psi(w, \cdot)$ is continuous and non-polynomial.
The desired result is an immediate consequence of  \cite[Theorem 1]{leshno1993multilayer}, which states that  feedforward networks with one hidden layer  and a locally bounded piecewise continuous activation function are universal function approximators if and only if the activation function is non-polynomial (almost everywhere). 
\end{proof}

Now let's consider a feedforward TransNN with a fixed global bias term $b\neq 0$ and with a  trainable activation-level parameter for each link (see Fig. \ref{fig:single-layer-TransNN}).

\begin{figure}[htb]
\centering
	\includegraphics[width=8cm]{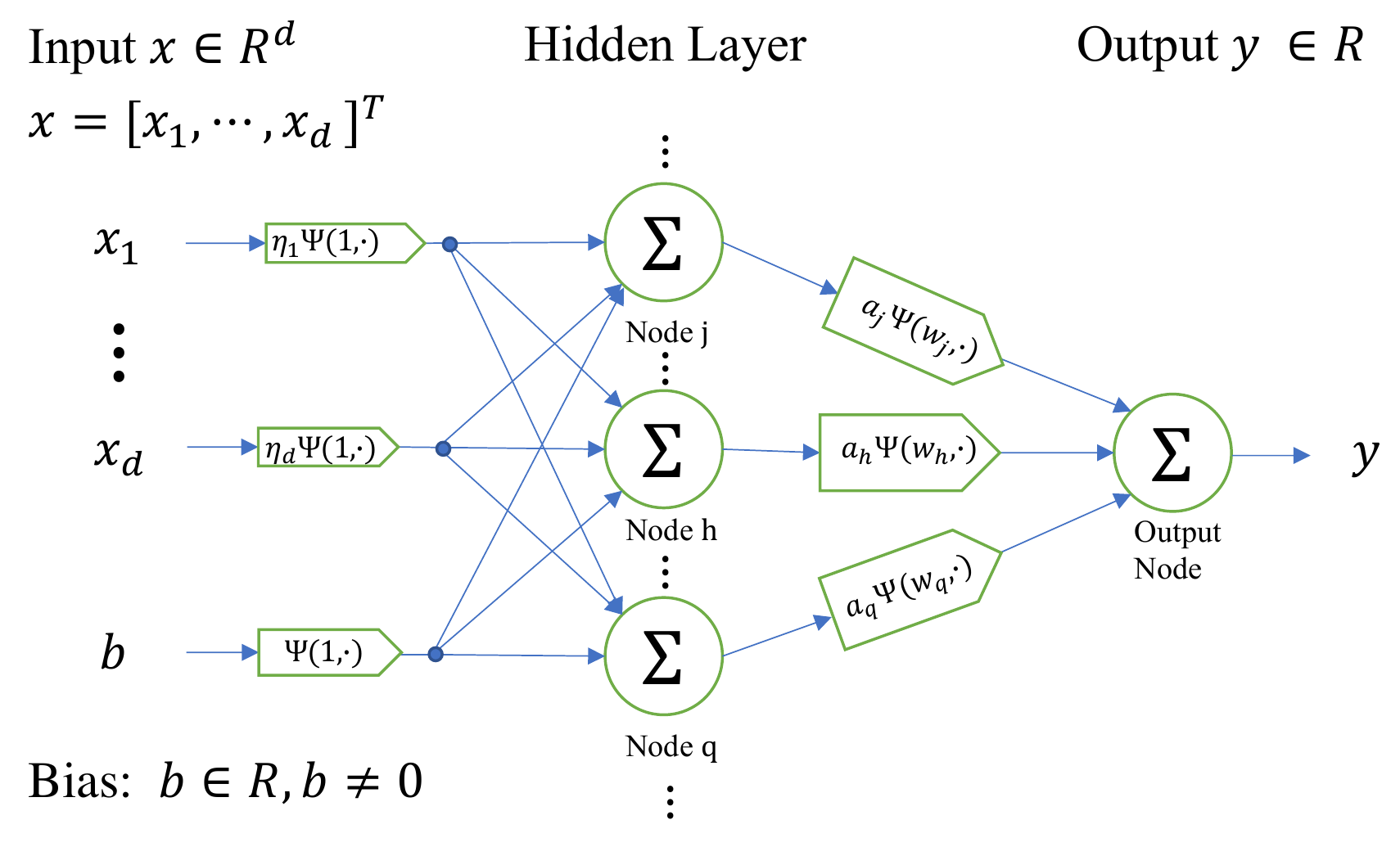}
	\caption{An illustration of single hidden layer TransNN with TLogSigmoid activation function $\Psi$. We note that $\Psi(1,\alpha)=\alpha$ for $\alpha \in \BR$.} \label{fig:single-layer-TransNN}
\end{figure}
\begin{theorem}[TransNNs with Real Weights]\label{thm:TransNNI-psi}
Let $b\neq 0$ be given.   Consider the parameter set  
\begin{equation}
 	\begin{aligned}
 		 \Theta_{_{\BR}}\triangleq \Big\{(n,  &(a_{i})_{i=1}^n, (\eta_i)_{i=1}^n, (w_i)_{i=1}^n)\big|  \\
 		 & n\in \BN, ~a_i  \in \BR, ~\eta_i\in \BR^n, ~w_i \in [0,1] \Big\}.
 	\end{aligned}
 \end{equation}
The feedforward neural network model with one hidden layer and a fixed non-zero bias term $b$ given by 
\begin{equation}  \label{eq:NNphi-selfloop}
\begin{aligned}
	y^\theta(x) & = \sum_{i=1}^n a_{i} \Psi(w_i, \eta_i^\TRANS x+b) ,  ~  x\in \BR^d, ~ y^\theta(x) \in \BR 
\end{aligned}
\end{equation}
with arbitrary parameters $\theta \triangleq (n, (a_{i})_{i=1}^n, (\eta_i)_{i=1}^n, (w_i)_{i=1}^n)$  in $\Theta_{_{\BR}}$ 
is a universal function approximator for $C(\BR^d)$. 
\end{theorem}
\begin{proof}%
We follow closely the proof of  \cite[Theorem 1]{leshno1993multilayer}. Let $b\neq 0$ be given. 
Let 
	\begin{equation*}
		{\sum_d}  \triangleq \text{span}\left\{\Psi(w, \eta^\TRANS x+b): ~\eta \in \BR^d,~ w\in [0,1] \right\}. 
	\end{equation*}

\vspace{0.1cm}	

\noindent\textbf{Step 1.} \emph{If $\sum_1$ is dense in $C(\BR)$, then $\sum_d$ is dense in $C(\BR^d)$}	

The space $V= \operatorname{span}\{f(a^\TRANS x)| a\in \BR^d, f\in C(\BR)\}$ is dense in $C(\BR^d)$ (see e.g. \cite{vostrecov1961approximation}).  Consider any arbitrary $g\in C(\BR^d)$ and any compact set $K\subset \BR^d$. $V$ is dense in $C(K)$ means that given any $\varepsilon >0$ there exists $f_i \in C(\BR)$ and $a_i\in \BR^d$, $i \in\{1,..., k\}$, such that 
$
|g(x)- \sum_{i=1}^k f_i(a_i^\TRANS x)|\leq {\varepsilon}/{2}
$ 
for all $x\in K.$	 Let $\{a_i^\TRANS x| x\in K\}\subset [\alpha_i, \beta_i]$ for finite interval $[\alpha_i, \beta_i]\subset \BR$, $i \in\{1,..., k\}$. 
Since $\Sigma_1$ is dense in $C([\alpha_i, \beta_i])$, $i \in\{1,..., k\}$,  there exist constants $c_{ij}\in \BR$, $w_{ij}\in [0,1]$ and  $\eta_{ij}\in \BR$, such that 
$
|f_i(y)-\sum_{j=1}^{m_i}c_{ij}\Psi(w_{ij},\eta_{ij}^\TRANS y+b)|< {\varepsilon}/{(2k)}
$
for all $y\in [\alpha_i, \beta_i]$. Hence, 
$
\big|g(x)- \sum_{i=1}^k\sum_{j=1}^{m_i}c_{ij}\Psi(w_{ij}, \eta_{ij}^\TRANS x+b)\big|<\varepsilon
$
for all $x\in K$. Thus $\Sigma_1$ is dense in $C(\BR)$ implies that $\Sigma_d$ is dense in $C(\BR^d)$. 
\vspace{0.1cm}	

\noindent\textbf{Step 2.} \emph{$\sum_1$ is dense in $C(\BR)$.}
\vspace{0.1cm}

For $w \in (0,1)$,  $\Psi(w, \cdot)$ is non-polynomial. Then
 $[\Psi(w, (\eta + h)x+b) - \Psi(w, \eta x+b) ]/h  \in \Sigma_1$ for every $\eta \in \BR$ with $h \neq 0$. Hence it follows that $(d/d\eta)\Psi(w, \eta x+b) \in \overline{\Sigma}_1$, where  $\overline{\Sigma}_1$ denotes the closure of $\Sigma_1$.  By the same argument, $(d^n/d\eta^n)\Psi(w, \eta x+b) \in \overline{\Sigma}_1$ for all $n\in \BN_0$ (and all $\eta \in \BR)$.
Let $ \Psi^{(n, 2)}$ denote the $n$th order partial derivatives with respect to the second variable of $\Phi(\cdot, \cdot)$. We note that
\begin{equation}
 \frac{d^n}{d\eta^n } \Psi(w, \eta^\TRANS x+b)  = x^n  \Psi^{(n, 2)}(w, \eta^\TRANS x +b). 
\end{equation}
Based on the explicit form of $ \Psi^{(n, 2)}$given in \eqref{eq:psi-grad-x} and \eqref{eq:high-grad-psi-form},  
we observe that there exists $\{\omega_n \in(0,1)\}_{n=1}^\infty$ such that %
\begin{equation}
\begin{aligned}
	x^n  \Psi^{(n, 2)} &(\omega_n, 0) =  \frac{d^n}{d\eta^n } \Psi(\omega_n, \eta x+b) \mid_{\eta =0}\quad
  \in 	\overline{\Sigma}_1	
\end{aligned}
\end{equation}
is always non-zero for all $n\geq 1$; furthermore, since $b\neq 0$, \text{when} $\omega_0 \in (0,1)$,
\begin{equation}
\Psi(\omega_0, \eta x+b)|_{\eta=0} = -\log(1-\omega_0+\omega_0 e^{-b})\neq 0  \in 	\overline{\Sigma}_1	 . 
\end{equation}
{This imply that   $\overline{\Sigma}_1$ contains all polynomials.} 
 By Weierstrass's Approximation Theorem, it follows that  $\overline{\Sigma}_1	$ contains $C(K)$ for each $K\subset R^n$. That is, ${\Sigma}_1	$ is dense in $C(R)$.
\end{proof}

\begin{theorem}[TransNNs with Rational Weights]\label{thm:TransNNI-psi-rational}
Let $b\neq 0$ be given.   Consider the parameter set  
\begin{equation}
 	\begin{aligned}
 		 \Theta_{_{\operatorname{Q}}}\triangleq \Big\{(n,  &(a_{i})_{i=1}^n, (\eta_i)_{i=1}^n, (w_i)_{i=1}^n)\big|  \\
 		 & n\in \BN,~ a_i  \in \operatorname{Q},~ \eta_i\in \BR^n, ~w_i \in [0,1] \Big\}.
 	\end{aligned}
 \end{equation}
Then the feedforward neural network model with one hidden layer, a fixed non-zero bias term $b$ and rational weights $\{a_i\}$ given by 
\begin{equation}  \label{eq:NNphi-selfloop}
\begin{aligned}
	y^\theta(x) = \sum_{i=1}^n a_{i} \Psi(w_i, \eta_i^\TRANS x+b) ,  \quad  x\in \BR^d, ~ y^\theta(x) \in \BR 
\end{aligned}
\end{equation}
with arbitrary parameters $\theta \triangleq (n, (a_{i})_{i=1}^n, (\eta_i)_{i=1}^n, (w_i)_{i=1}^n)$ in $\Theta_{_{\operatorname{Q}}}$, is a universal function approximator for $C(\BR^d)$.
\end{theorem}
\begin{proof}
Since the set of rationals are dense in the set of reals, polynomials with rational coefficients are dense in polynomials with real coefficients. This, together with  Theorem \ref{thm:TransNNI-psi}, implies the desired result.
\end{proof}
\begin{theorem}[$m$-Output TransNNs with Rational Weights]\label{thm:TransNNIRm} 
Let $b\neq 0$ be given. 
Consider a feedforward neural network $F^\theta(x)=[F_1^\theta(x), ..., F^\theta_m(x)]^\TRANS$  from $\BR^d \to \BR^m$ given as follows:  \begin{equation}  \label{eq:NNphi-selfloop2}
\begin{aligned}
	& F_i^\theta(x) = \sum_{j=1}^n a_{ij} \Psi(\omega_{j}, \eta_j^\TRANS x+b), ~ ~ x\in \BR^d, ~F_i^\theta \in \BR
	 \end{aligned}
\end{equation}
for $i \in [m]$,  where the parameter $$\theta \triangleq (n, (a_{ij})_{i\in[m],j \in[n]}, (\omega_j)_{j\in[n]},(\eta_j)_{j\in [n]})$$ can be arbitrarily chosen from  the parameter set
\begin{equation}
 	\begin{aligned}
 		 \Theta_{_{\operatorname{Q}}}^m\triangleq  \Big\{(n, &(a_{ij})_{i\in[m],j \in[n]}, (\omega_j)_{j\in[n]},(\eta_j)_{j\in [n]}))\big|  \\
 		 & n\in \BN,~ a_{ij}  \in \operatorname{Q},~ \eta_j\in \BR^n, ~w_j \in [0,1] \Big\}.
 	\end{aligned}
 \end{equation}
Then the neural network model $F^\theta(\cdot):\BR^d\to \BR^m$ in  \eqref{eq:NNphi-selfloop2}  is a universal function approximator for  $C(\BR^d;\BR^m)$. %
\end{theorem}
\begin{proof}
	Let the bias term $b\neq 0$ be given and let 
	$$
	\begin{aligned}
			\sum_{d, \operatorname{Q}}  \triangleq \Big\{ \sum_{i=1}^n &a_i\Psi(w, \eta^\TRANS x+b): n\in \BN,  \\
			& a_i \in \operatorname{Q}, w \in [0,1], ~\eta \in \BR^d \Big\}. 
	\end{aligned}
	$$ %
Since Theorem \ref{thm:TransNNI-psi-rational} holds, we only need to prove that the density of $\sum_{d, \operatorname{Q}}$ in $C(\BR^d)$ implies the density of the set of all functions  characterized by \eqref{eq:NNphi-selfloop2} in $C(\BR^d; \BR^m)$.  The proof is as follows.
Consider an arbitrary continuous function $g(\cdot):\BR^d \to \BR^m$. Then  $g(\cdot)$ can be represented  by 
$
g(x)  = [g_1(x),\cdots, g_m(x)]^\TRANS, ~ x\in \BR^d.
$
Since $\Sigma_{d, \operatorname{Q}}$ is dense in $C(\BR^d)$, 
for 
${\varepsilon >0}$, any compact subset of ${K\subseteq \BR^{d}}$ and any 
${g_i\in C(K)}$, there exists 
${F_i\in \Sigma_{d, \operatorname{Q}}}$ such that
$${\operatorname{ess}\sup_{x\in K}|F_i(x)-g_i(x)|<\frac{\varepsilon}{\sqrt{m}} }.$$
Let $F(x)= [F_1(x),..., F_m(x)]^\TRANS$. 
Then we obtain that
\[
\begin{aligned}
	\operatorname{ess}\sup_{x\in K} \|F(x)- g(x)\|^2 & = \operatorname{ess}\sup_{x\in K} \sum_{i=1}^m (F_i(x)-g_i(x))^2 
	\leq  \varepsilon^2. 
\end{aligned}
\]
This immediately implies $\operatorname{ess}\sup_{x\in K} \|F(x)- g(x)\|\leq \varepsilon  $. Since the choice of $K$ is arbitrary, we have the desired result.
\end{proof}
\begin{remark}
 Results in  
  Theorems \ref{thm:TransNNI-psi}-\ref{thm:TransNNIRm}   
 still hold if the TLogSigmoid function activation $\Psi(\cdot, \cdot)$ is replaced by TLogSigmoidPlus $\Psi_{+}(\cdot, \cdot)$ in \eqref{eq:PsiPlus} along with a positive bias term $b>0$,  or by TSoftAffine $\Phi(\cdot,\cdot)$ in \eqref{eq:phi-activation}. %
\end{remark}
\begin{remark}[Arbitrary Depth]
Let standard feedforward neural networks with $d$ input nodes, $m$ output nodes, and an arbitrary number of layers, each of which has $k$ nodes with activation function $\rho$ (which allow bias terms) be denoted by  $\mathcal{NN}_{d,m,k}^\rho$.  We observe that for any $w\in(0,1)$, the activation functions $\Psi(w,\cdot), \Psi_{+}(w,\cdot)$ and $\Phi(w,\cdot)$ are all continuous non-polynomial functions which are continuously differentiable  with nonzero derivative on a non-empty set of points.  
Applying the uniform approximation property for neural networks with arbitrary depth in 
	\cite[Proposition 4.9]{kidger2020universal}, we obtain that for any fixed $w\in (0,1)$ and for any compact $K\subset \BR^d$, the neural networks with arbitrary depth
$\mathcal{NN}_{d,m,d+m+1}^{\Psi(w,\cdot)}$, $\mathcal{NN}_{d,m,d+m+1}^{\Psi_+(w,\cdot)}$, and $\mathcal{NN}_{d,m,d+m+1}^{\Phi(w,\cdot)}$ are all dense in $C(K;\BR^m)$ with respect to the uniform norm.
\end{remark}
\section{continuous Time TransNN Models Give Rise to the Network SIS Model} \label{sec:netowrk-sis-via-TransNN}
%
 In this section we derive continuous time TransNNs from discrete time  virus spread models  in \eqref{eq:probablistic-dynamics} and \eqref{eq:population-model}  with extra assumptions on the transmission probability rate over the time duration. Based on the continuous time TransNNs, we derive the standard continuous time network SIS model in \cite{lajmanovich1976deterministic,van2008virus}.  
\subsection{Single-Particle Transmission Model}

 Assume the cross-node transmission probability is roughly linear in a small time duration $\Delta\geq 0 $, and assume the self-transmission rate may be exponential over the time duration;
more specifically,
\begin{equation}\label{ass:single-virus-ct}
	\begin{aligned}
	\textbf{Assumption:}\quad 
&w_{ij} =   c_{ij} \Delta + o(\Delta), \quad \forall i\neq j, \quad  \\
&   w_{ii}=e^{-c_{ii}\Delta}  = 1 - c_{ii}\Delta + o(\Delta), 
\end{aligned}
\end{equation}
 for all $i, j \in [n]$,  where $c_{ij}\geq 0$ is the basic transmission probability rate  (per unit time) from node $j$ to node $i$, and $c_{ii}\geq 0$ is the self-healing probability rate  (per unit time). 

 The exponential recovery rate is used in the first work on network SIS model \cite{lajmanovich1976deterministic}. The relation $w_{ii}= e^{-c_{ii}\Delta}$ has the properties that  when $\Delta =0$, $w_{ii}=1$ and moreover, in the long run as $\Delta \to +\infty$, the probability of transmitting to itself $w_{ii}$ is zero  (that is, nodes can heal themselves in the long run). However, it is not the only choice; we can use any other self-transmission rate that satisfies $w_{ii} =  1-c_{ii}\Delta + o(\Delta)$ and it leads to the same continuous time SIS network model (which will be derived later). 

Following the virus spread model in \eqref{eq:probablistic-dynamics} and, in particular, its equivalent TransNN representation in \eqref{eq:nn-individual-node}, we obtain 
\[
\begin{aligned}
	s_i(t+\Delta) = & \sum_{j=1, j\neq i}^n a_{ij} \Psi( c_{ij} \Delta + o(\Delta), s_j(t) )  \\
	& + \Psi(e^{-c_{ii}\Delta}, s_i (t)) .
\end{aligned}
\]
 The rate of the state variation over the time duration $\Delta$ satisfies
\[
\begin{aligned}
	&\frac{s_i(t+\Delta) - s_i(t)}{\Delta}\\
	& =   \frac{\sum_{j=1,j\neq i}^n a_{ij} \Psi( c_{ij} \Delta  + o(\Delta), s_j(k) ) }{\Delta}\\
	& \quad +\frac{\Psi(e^{-c_{ii}\Delta}, s_i (t)) -  s_i (t)}{\Delta}\\
	& = \frac{\sum_{j=1,j\neq i}^n a_{ij} \left( \frac{ 1-e^{-s_j(k)}}{e^{-\Psi({0}, s_j(k))}} c_{ij} \Delta + o(\Delta)\right)}{\Delta}\\
	& \quad + \frac{s_i(t) + \frac{1- e^{-s_i(k)}}{e^{-\Psi({1}, s_i(k))}} \cdot (-c_{ii}) e^{-c_{ii}\Delta} \cdot  \Delta +o(\Delta) -s_i(t)}{\Delta}
\end{aligned}
\]
where  the last step  uses the partial derivative of $\Psi(w, \cdot)$ with respect to its first element $w$ given by \eqref{eq:psi-1stgradient}.
Taking the small time limit $\Delta \to 0$ and utilizing that $
\lim_{w \to 0} \log(1-w+ w e^x) = 0  
$ yield  the associated \emph{Continuous Time TransNN Model with Single Particle Transmissions} under the assumptions in \eqref{ass:single-virus-ct}: 
\begin{equation}
\begin{aligned}
	\frac{ds_i(t)}{dt} \triangleq& \lim_{\Delta \to 0}  \frac{s_i(t+\Delta) - s_i(t)}{\Delta}  \\
	 = &  {\sum_{j=1,j\neq i}^n a_{ij}   c_{ij} { (1-e^{-s_j(t)})}  } - c_{ii}\frac{1-e^{-s_i(k)}}{e^{-s_i(k)}}.
\end{aligned}
\end{equation}
Let $q_i(t) \triangleq 1-p_i(t)$ which represents the probability of node $i$ being healthy. Then $s_i(t) = - \log q_i(t)$ and hence we obtain
\[
-\frac{d \log q_i (t)} {dt} =  {\sum_{j=1, j\neq i}^n a_{ij}   c_{ij} {({1-q_j(t)})}  } - (1-q_i(t)) \frac{c_{ii}}{ q_i(t)}.
\]   
Using $p_i(t) = 1-q_i(t)$ then yields 
\begin{equation}\label{eq:network-sis-1st}
	\frac{d p_i (t)} {dt} =   (1-p_i(t)){\sum_{j=1,j\neq i}^n a_{ij}   c_{ij} { {p_j(t)}} -{c_{ii} p_i(t) }}
\end{equation}
where $p_i(t)$ is the probability of infection at node $i$ at time $t$.

The equation \eqref{eq:network-sis-1st}  is essentially the network SIS models proposed in \cite{lajmanovich1976deterministic,van2008virus} (with slightly different interpretations of the connections, nodes and nodal states). 
It is worth highlighting that this new derivation of network  SIS model which starts from the virus spread model  in \eqref{eq:probablistic-dynamics}  
provides deeper fundamental understanding of virus spread models on networks.
Furthermore, specializing $c_{ij}$ to a single constant $c$  in \eqref{eq:network-sis-1st} yields a demonstration that the discrete time virus spread model  in \cite{chakrabarti2008epidemic} is consistent with the continuous time network SIS models in \cite{lajmanovich1976deterministic,van2008virus}. 
\begin{remark}[Network SI Model]
	We note that if $c_{ii}=0$, then the model is a network SI model where individuals once infected stay infected forever. Such a model is useful for modelling the spread of incurable infectious diseases. 
\end{remark}
\subsection{Multiple-Particle Transmission Model}
In this section, we derive the continuous time model for the virus spread dynamics with multiple particle transmissions  at each link in \eqref{eq:population-model}. 
We assume that
\begin{equation}\label{ass:pop-virus-ct}
\begin{aligned}  \textbf{Assumption:}\qquad 
	 & 
	a_{_{hq}} = \Delta^\varepsilon c_{_{hq}} + o(\Delta^\varepsilon), \quad \\
	 &  a_{_{hh}} = 1 - \Delta^\varepsilon c_{_{hh}} + o(\Delta^\varepsilon),\quad\\
	& w_{_{hq}} = \kappa_{_{hq}} \Delta^{1-\varepsilon} + o(\Delta^{1-\varepsilon}) , \quad  \\
	& w_{_{hh}} = 1  -\kappa_{_{hh}} \Delta^{1-\varepsilon} + o(\Delta^{1-\varepsilon}),
\end{aligned}
\end{equation}
 for some $\varepsilon\in[0,1]$ and for all $h,q \in [n]$, where $c_{_{hq}}\geq 0$ and $\kappa_{_{hq}}\geq 0$. Different $\varepsilon \in[0,1]$ may be chosen depending  on application contexts and interpretations of $a_{_{hq}}$ and $w_{_{hq}}$.  
Then applying these assumptions to the equivalent TransNN representation in \eqref{eq:pop-nn} of the virus spread model  with multiple transmission particles at each link in \eqref{eq:population-model}, we obtain
\begin{equation}
	\begin{aligned}
	&s_h({t+\Delta})  \\
	& = 	\sum_{q=1, q\neq h}^n (\Delta^\varepsilon   c_{_{hq}} + o(\Delta^\varepsilon )) \Psi(\Delta^{1-\varepsilon} \kappa_{_{hq}}+o(\Delta^{1-\varepsilon}) , s_q(t)) 	\\
	& \quad +  (1  - \Delta^\varepsilon  c_{_{hh}} + o(\Delta^\varepsilon ))  \Psi(1  -\kappa_{_{hh}} \Delta^{1-\varepsilon} + o(\Delta^{1-\varepsilon}) , s_q(t)) \\
	& = \sum_{q=1 , q\neq h}^n (\Delta^\varepsilon   c_{_{hq}} + o(\Delta^\varepsilon )) \Big[\Psi(0 , s_q(t))  	\\
	& \quad \quad + \frac{1-e^{-s_q(t)}}{e^{-\Psi({0}, s_q(t))}}  \kappa_{_{hq}} \Delta^{1-\varepsilon}  + o(\Delta^{1-\varepsilon})\Big]\\
	& \quad +  (1  - \Delta^\varepsilon  c_{_{hh}} + o(\Delta^\varepsilon ))    \Big[\Psi(1  , s_q(t)) \\
	& \qquad + \frac{1-e^{-s_q(t)}}{e^{-\Psi({1 }, s_q(t))}} (-\kappa_{_{hh}})\Delta^{1-\varepsilon}  + o(\Delta^{1-\varepsilon})\Big],
	\end{aligned} 
\end{equation}
where the last step  uses the partial derivative of $\Psi(\cdot, \cdot)$ with respect to its first element given by \eqref{eq:psi-1stgradient}. Using the property   \[
  \lim_{w \to 0} \Psi(w, x) = (-1) \lim_{w \to 0} \log(1-w+ w e^{-x}) = 0,
\]
and taking the small time limit $\Delta \to 0$, 
we  obtain
the associated \emph{Continuous Time TransNN Model with Multiple Particle Transmissions} under the assumptions in \eqref{ass:pop-virus-ct}:
\begin{equation}
\begin{aligned}
		\frac{d s_h(t)}{d t}		&= \sum_{q=1, q\neq h}^n c_{_{hq}} \kappa_{_{hq}} (1-{e^{-s_q(t)}}) -c_{_{hh}}\kappa_{_{hh}} \frac{1- e^{-s_h(t)}}{e^{-s_h(t)}} .
\end{aligned}
\end{equation}
Let $\rho_h(t) \triangleq 1- p_h(t)$. Then $s_h(t) = -\log \rho_h(t)$, and hence  
\[
-\frac{d \log \rho_h}{dt} = -c_{_{hh}}\kappa_{_{hh}} \frac{(1-\rho_h)}{\rho_h}+   \sum_{q=1, q\neq h}^n c_{_{hq}} \kappa_{_{hq}} ({1-\rho_q})
\]
 Using $p_h(t)=1- \rho_h(t) $  then yields
\begin{equation} \label{eq:mult-tran-netsis}
	\frac{d p_h}{dt}= -c_{_{hh}}\kappa_{_{hh}} p_h +  (1-p_h) \sum_{q=1, q\neq h}^n c_{_{hq}} \kappa_{_{hq}} p_q(t)
\end{equation}
where $c_{_{hq}}$ denotes the transport rate (of particles) per unit time from node $q$ to node $h$, $\kappa_{_{hq}}$ denotes the infection probability rate (per unit time) pre particle from node $q$ to node $h$, $\kappa_{_{hh}}$ denotes the self-healing rate (per unit time) per particle for node $h$.

{An interesting and important feature in the dynamics in \eqref{eq:mult-tran-netsis} is that $c_{_{hq}} $ and $\kappa_{_{hq}}$ are multiplied together, and clearly the quantity $c_{_{hq}}\kappa_{_{hq}} $ represents the expected number of effective particle transmissions (that cause the infection of node $h$) from node $q$ to node $h$, per unit time. }

\section{Conclusion}
This work connects virus spread models with their equivalent neural network representations. Based on this connection, we propose Transmission Neural Networks (TransNNs) as a new learning architecture and identifies a set of tunable or trainable activation functions (including TLogSigmoid, TLogSigmoidPlus, and TSoftAffine activation functions). Moreover, TransNNs also facilitate our new fundamental derivations of standard network SIS epidemic models. 
 Different choices of  the ``states" of the same system lead to different representation models for spread dynamics (which could be for neural networks, neuronal networks, epidemic networks and rumour spread networks, etc.); more specifically, choosing the negative-log-negative probability state (i.e. the Shannon information) gives rise to neural network models (i.e. TransNNs), and using the probability state leads to  virus spread models.  
 
This work opens up many interesting and important future directions, and we mention some in the following.

(a) It is important to develop a control theory or methodology for systems characterized by TransNNs as such a development naturally have important implications in controlling epidemic spreads and modifying system level properties of neural networks as learning and inference models.

(b) Since TransNNs relate neural networks to Markov models,  there is a potential connection  between the TransNN models and reinforcement learning (based on Markov decision processes with unknown dynamics and rewards), and such a connection shall be explored. 

(c) We should explore TransNNs with realizations of the probabilistic connections, which will result in the random network characterizations of TransNNs, and the associated properties and applications should be investigated. Furthermore, the probability state can be also characterized via the stochastic realizations of binary states similar to those in deep belief networks (see. e.g. \cite{neal1990learning,hinton2006fast}). The associated inference and learning problems shall be investigated. 

(d) Biological neuronal excitations on chemical synaptic networks seem naturally fit in TransNN models with multiple particle transmissions.  TransNN characterizations and experimentations for biological neuronal excitations taking into account of inhibitory neurons shall be explored. 

(e) The activation-level parameter enables  interesting characterizations of  activation intensity of individual neurons and individual connections.  Such heterogeneous activation levels may be potentially useful in improving the training performance and the robustness of TransNNs as learning models, and furthermore seem useful in representing neuronal models with neural modulations. Moreover, whether individual activations per link leads to better learning models is a research question that needs to be addressed thoroughly. 

(f) The tunable activation functions (e.g. TLogSigmoid, TLogSimoidPlus and TSoftAffine) with activation levels outside $[0,1]$ may lead to interesting activation functions as well.  In addition, these activation functions shall be integrated with all the existing neural network  models. For instance, recurrent neural networks with these activations, along with universal function approximator properties,  shall be investigated. %

(g) Deep learning tools and numerical methods are readily used to train TransNNs and hence their associated virus spread models. Thus such learning tools and methods,  together with TransNNs, should be explored to estimate and predict epidemic spreads  using real-world epidemic data. 

(h) Since the derivation of the standard network SIS models characterized by differential equations via TransNNs improves our fundamental understandings of epidemic models on networks,  the derivation of compartmental epidemic models on networks with more states  (such as SIR and SEIR) and extra features (such as location and age) shall be rigours formulated and analyzed following similar ideas in this paper. 
\section*{Acknowledgment}

The first author would like to thank Mi Lin for discussions and  valuable feedback on this work since its start.

 \appendices
\section{Proof of Theorem \ref{eq:threshold} } \label{sec:threshold-proof}

\begin{proof}
We note that \eqref{eq:probablistic-dynamics} and \eqref{eq:nn-individual-node} are equivalent models, and the mapping  $s_i(k) = -\log(1-{p_i}(k))$ from  the state $p_i(k)$ of \eqref{eq:probablistic-dynamics} to the state $s_i(k)$ of \eqref{eq:nn-individual-node} is a monotone bijection that satisfies $s_i(k)=0$ if and only if $p_i(k)=0$ for all $i\in[n]$ and for all $k\in \BN_0$. Hence, without loss of generality, we only analyze the stability of the model \eqref{eq:nn-individual-node}.  
We observe that $\Psi(w,x)$ is concave in $x\in [-\infty,+\infty]$, since its 2nd order derivative with respect to $x$ is non-positive (as shown later in \eqref{eq:2nd-Psi-grad-w}). 
Thus, for any  $w\in [0,1]$, 
\[
 \Psi(w, z)\leq \Psi(w, x)+\partial_x \Psi(w,x)(z-x),\quad \forall x, z\in [-\infty, +\infty]. 
\]
Applying this property to the virus spread model \eqref{eq:nn-individual-node} yields
\begin{equation*}
\begin{aligned}
		& s_{i}(k+1)\leq   \sum_{j=1}^n a_{ij} \left( \Psi( w_{ij},s_j^*)+ \partial_x \Psi( w_{ij},s_j^*)(s_j(k)-s^*_j)\right)
\end{aligned}
\end{equation*}
for any $s^*=[s^*_1,..., s_n^*] \in \BR^n$. Furthermore, it is easy to verify that the state $0 \in \BR^n$ (corresponding the state of no infection) is a  fixed point of the TransNN model \eqref{eq:nn-individual-node}. %
Taking $s^* = 0$ yields
\begin{equation}
		 s_{i}(k+1)\leq   \sum_{j=1}^n a_{ij}  w_{ij}s_j(k), \quad i \in [n].
\end{equation}
since $\Psi( w_{ij},0)=0$ and $\partial_x \Psi( w_{ij},0) =w_{ij}$ (as shown later in Section \ref{sec:Psi-Activation}). 
Applying the standard stability condition for discrete time linear systems, we obtain a sufficient condition for the system above to be globally exponentially stable:
$	\max_{i\in[n]}|{\lambda}_i(A\odot W)| <1 $.
\end{proof}

\section{TSoftAffine Activation Function $\Phi$}\label{sec:Phi-activation}
If we take the log-negative-probability as the state for that virus spread dynamics in \eqref{eq:probablistic-dynamics} and \eqref{eq:population-model}, that is, 
\[
\bar{s}_i = \log(1-p_i),  \quad \bar{s}_i \in [-\infty, 0],
\]
we obtain the same corresponding TransNNs in \eqref{eq:nn-individual-node} and \eqref{eq:pop-nn}, except with a different activation function $\Phi$ and different state ranges. More specifically, 
\begin{equation}
	\bar{s}_{i}(k+1) = \sum_{j=1}^n a_{ij} \Phi( w_{ij},\bar{s}_j(k)), \quad i \in [n], ~k \in \BN_0
\end{equation}
with $\bar{s}_i \in [-\infty, 0]$ for $i \in [n]$.
The  TSoftAffine activation function denoted by  $\Phi$ is the double reflections (vertically and horizontally) of TLogSigmoid function $\Psi(w,\cdot)$, that is,  $\Phi(w,\cdot) = -\Psi(w,-\cdot)$, 
and it is  explicitly given by
\begin{equation}
	\Phi(w, x) = \log \left(1- w + we^{x} \right), \quad w \in [0,1], ~ x\in \BR.
\end{equation}
See Fig. \ref{fig:activation-phi} for the shape of TSoftAffine activation function $\Phi$ with different activation levels. 
\subsection{Derivatives of TSoftAffine $\Phi(w,x)$ with Respect to $w$}
The partial derivative of $\Phi(w, x)$ with respect to $w \in [0,1]$ satisfies
\begin{equation} \label{eq:phi-1stgradient}
	\partial_w{\Phi({w}, x)}= \frac{e^x-1}{1- w +w e^x } = \frac{e^x-1}{e^{\Phi(w,x)}}
\end{equation}
for any $ x\in \BR$.
{The denominator $1+w(e^x -1)$ is non-zero for any $x\in(-\infty, +\infty]$ and any $w\in [0,1]$.   The only singular point of the derivative is when $w=1$, the derivative at $x=-\infty$ is infinite.}

Higher order partial derivatives of $\Phi$ with respect to $w \in(0,1)$ are given as follows: for $k\geq 1$ and $x\in \BR$,
\[
\begin{aligned}
	\partial^k_w{\Phi({w}, x)} & = (-1)^k(k-1)!\frac{(e^x-1)^k}{e^{k\Phi(w, x)}} \\
	& = (-1)^k(k-1)! \left(\partial_w{\Phi({w}, x)}\right)^k , ~w\in (0,1). 
\end{aligned}
\]
We note that ${\partial_w^k}{\Phi({w}, 0)}=0$ for all $w\in[0,1]$ and $k\geq 1$.
\subsection{Derivatives of  TSoftAffine $\Phi(w,x)$ with Respect to $x$} \label{sec:high-grad-phi}
The partial derivative of $\Phi(w,x)$ with respect to $x \in [-\infty,+\infty]$ satisfies
\begin{equation}\label{eq:grad-x-phi}
	\partial_x \Phi(w, x) = \frac{w e^{x}}{1 - w +w e^{x}} = \frac{w e^{x}}{e^{\Phi(w, x)}} 
\end{equation}
for any  $w \in[0,1]$.
The second order partial derivative of $\Phi(w,x)$ with respect to $x \in [-\infty, +\infty]$ is given by 
\begin{equation}\label{eq:phi-2grad-x}
	\partial^2_x\Phi(w,x) =  \partial_x \Phi(w, x)(1- \partial_x \Phi(w, x)),
\end{equation}
for any $w\in [0,1]$.
The property of $ \partial_x \Phi(w, x)$ in \eqref{eq:phi-2grad-x}  resembles the property of the sigmoid function $\sigma$  that $\sigma^\prime = \sigma(1-\sigma)$. 
Adapting the analysis in  \cite{minai1993derivatives} for  the sigmoid function $\sigma$ to $ \partial_x \Phi(w, x)$, 
we obtain higher order partial derivatives of $\Phi(w,x)$ with respect to $x \in[-\infty, +\infty]$  as follows:  for  $~n \geq 2$, 
\begin{equation}\label{eq:high-grad-phi}
	{\partial^n_x }\Phi(w,x) =\sum_{k=1}^{n} (-1)^{k-1} (k-1)! S_{n, k}(\partial_x \Phi(w, x))^k
\end{equation}
for any $w\in[0,1]$
where $S_{n,k}$  denotes the Stirling numbers of the second kind (see e.g. \cite[Chapter 6.1]{graham1989concrete}) and
	$\partial_x \Phi(w, x) $ is given by \eqref{eq:phi-1stgradient}.

By the relation $\Psi(w,x)= -\Phi(w, -x)$ between TLogSimoid and TSoftAffine given in \eqref{eq:phi-activation}, 
we obtain that for $n\geq 1$, $w\in[0,1]$,
\begin{equation*}
	\begin{aligned}
		{\partial^n_x }\Psi(w,z) & = (-1)^{n+1}{\partial^n_x }\Phi(w,-z), \quad  z\in [-\infty +\infty].
		 	\end{aligned}
\end{equation*}%
This together with \eqref{eq:high-grad-phi} implies that higher order partial derivatives of TLogSigmoid function $\Psi(w,x)$ with respect to $x \in[-\infty, +\infty]$ are given by
\begin{equation*}
	\begin{aligned}
		&\partial^n_x \Psi(w,z)   = \sum_{k=1}^{n} (-1)^{k+n} (k-1)! S_{n, k}(\partial_x \Phi(w, {-z}))^k \\
		& = \sum_{k=1}^{n} (-1)^{k+n} (k-1)! S_{n, k}(\partial_x \Psi(w, {z}))^k, \quad z \in [-\infty, +\infty], 
	\end{aligned}
\end{equation*}
where the last equality is due to $\partial_x \Psi(w,z) = \partial_x \Phi(w,{-z})$. %

\subsection{Relations with Sigmoid, Tanh and Softplus Functions}

Interestingly,  TSoftAffine activation function is related to the Softplus function, with ${ \text{Softplus}(x)=\ln(1+e^{x}) }$,  that was first derived as the primitive function of the sigmoid function in \cite{dugas2001incorporating}.  More specifically,  $\Phi(0.5, \cdot)$ give rise to the Softplus function with an offset $\log(0.5)$, that is, for any $x\in \BR$,  
\begin{equation}
	\Phi(0.5, x) = \log(0.5 e^{x}+0.5) = \log(0.5)+ \log(1+e^x). 
\end{equation}

Moreover, the two partial derivatives of TSoftAffine $\Phi(\cdot, \cdot)$ are related to the sigmoid function $ \sigma(x)=\frac{1}{1+e^{-x}} $ and the hyperbolic tangent function $\tanh(x) = \frac{e^{x}-e^{-x}}{e^{x}+e^{-x}}$  as follows: for any $x\in \BR$,  
\begin{equation}
\begin{aligned}
		& \partial_x \Phi (0.5, x) = \sigma(x) ~~ \text{and}~~
		 \partial_w \Phi(0.5, x ) = 2\tanh (0.5 x).
\end{aligned}
\end{equation}
An important feature of $\partial_x \Phi (w, x)$ is as follows: 
$$\partial_x\Phi(w, 0)=w, ~ \partial_x\Phi(-\infty, 0)=0  ~ ~\text{and}~~ \partial_x\Phi(+\infty, 0)=1.$$ Therefore, we see that $\partial_x\Phi(w, \cdot)$ is actually a tunable sigmoid function (TSigmoid) with threshold value $w$ (see Fig. \ref{fig:grad-x-phi}). 
\begin{figure}[htb]
\centering
	\includegraphics[width=8.5cm]{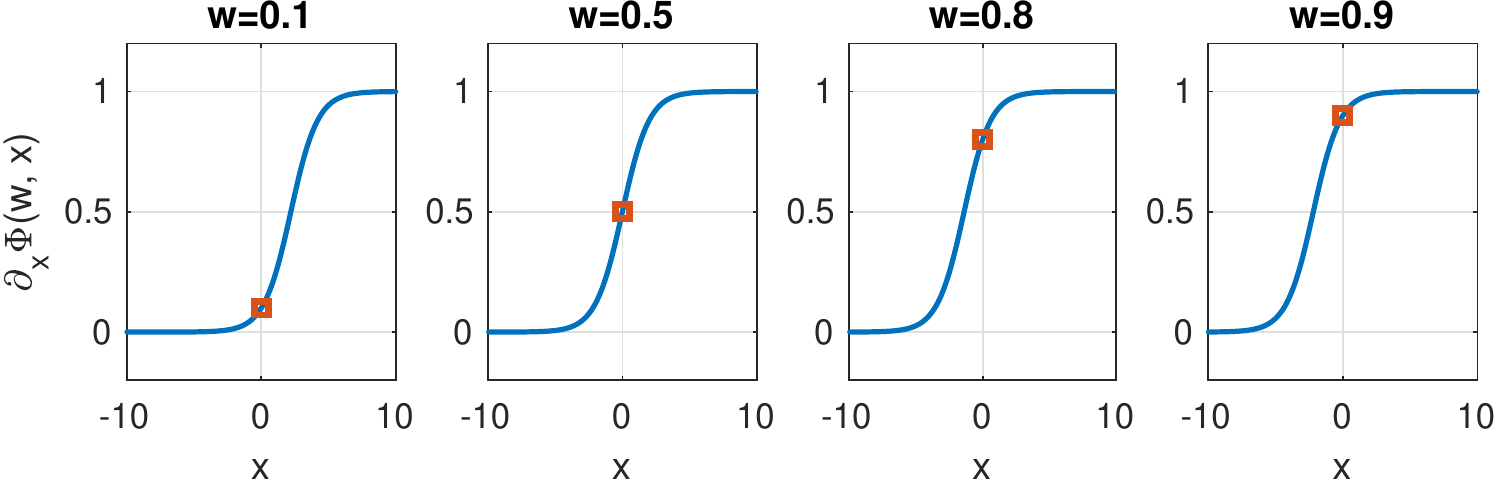}
	\caption{Illustration of $\partial_x\Phi(w, \cdot)$ defined in \eqref{eq:grad-x-phi} as a tunable sigmoid function with threshold value $w$.  The intersection value  of $\partial_x\Phi(w, \cdot)$ with the vertical axis is $w$ as highlighted by red squares.} \label{fig:grad-x-phi}
\end{figure}
\bibliographystyle{IEEEtran}
\bibliography{mybib}
\begin{IEEEbiography}{Shuang Gao}(M'19) received the B.E. degree in automation  and M.S. in control science and engineering, from Harbin Institute of Technology, Harbin, China, respectively in 2011 and 2013. He received the PhD degree in electrical engineering (with specialization in system and control) from McGill University, Montreal, QC, Canada, in 2019, under the supervision of  Peter. E. Caines. 
He is currently a Postdoctoral Researcher associated with the Department of Electrical and Computer Engineering at McGill University, the McGill Centre for Intelligent Machines, and Groupe d'\'Etudes et de Recherche en Analyse des D\'ecisions in Montreal, Canada. He is also a visiting scholar in the School of Mathematics and Statistics at Carleton University hosted by  Minyi Huang since December 2021 and an incoming research fellow at the Simons Institute for the Theory of Computing at UC Berkeley starting from August 15, 2022. 
His current research interests include control, game and learning theories for large networks, and their applications in social networks, epidemic networks,  renewable energy grid and neuronal networks.

\end{IEEEbiography}
\begin{IEEEbiography}{Peter E. Caines} (LF'11)
received the BA in mathematics from Oxford University in 1967 and the PhD in systems and control theory in 1970 from Imperial College, University of London, under the supervision of David Q. Mayne, FRS. After periods as a postdoctoral researcher and faculty member at UMIST, Stanford, UC Berkeley, Toronto and Harvard, he joined McGill University, Montreal, in 1980, where he is Distinguished James McGill Professor and Macdonald Chair in the Department of Electrical and Computer Engineering. In 2000 the adaptive control paper he coauthored with G. C. Goodwin and P. J. Ramadge (IEEE Transactions on Automatic Control, 1980) was recognized by the IEEE Control Systems Society as one of the 25 seminal control theory papers of the 20th century. He is a Life Fellow of the IEEE, and a Fellow of SIAM, IFAC, the Institute of Mathematics and its Applications (UK) and the Canadian Institute for Advanced Research and is a member of Professional Engineers Ontario. He was elected to the Royal Society of Canada in 2003. In 2009 he received the IEEE Control Systems Society Bode Lecture Prize and in 2012 a Queen Elizabeth II Diamond Jubilee Medal. Peter Caines is the author of Linear Stochastic Systems, John Wiley, 1988, republished as a SIAM Classic in 2018, and is a Senior Editor of Nonlinear Analysis-Hybrid Systems; his research interests include stochastic, mean field game, decentralized and hybrid systems theory,  together with their applications in a range of fields.
\end{IEEEbiography}

\end{document}